\documentclass{article}

\usepackage{microtype}
\usepackage{graphicx}
\usepackage{subcaption}
\usepackage{booktabs} 

\usepackage{hyperref}

\usepackage[preprint]{icml2026}

\usepackage{amsmath}
\usepackage{amssymb}
\usepackage{mathtools}
\usepackage{amsmath} 
\usepackage{colortbl} 

\usepackage{booktabs} 
\usepackage{bm} 

\definecolor{lightblue}{rgb}{0.81, 0.94, 1.0}
\definecolor{cvprblue}{rgb}{0.21,0.49,0.74}
\definecolor{aligncol}{RGB}{20,20,20} 
\definecolor{repelcol}{RGB}{120,120,120} 

\usepackage{hyperref}       %
\hypersetup{
    colorlinks=true,
    linkcolor=red,
    citecolor=cvprblue,      
    urlcolor=cyan
}
\usepackage[capitalize,noabbrev]{cleveref} 
\usepackage{tikz}  
\usetikzlibrary{arrows.meta,calc,angles,quotes,matrix,positioning,shapes.geometric}
\tikzset{>=Stealth}

\usepackage{float}
\usepackage{wrapfig}
\usepackage{lipsum} 
\usepackage{subcaption}
\usepackage{multirow}
\usepackage{derivative}
\usepackage{latexsym}
\usepackage{amssymb} 
\usepackage{pifont} 
\newcommand{\cmark}{\ding{51}}
\usepackage{nicefrac}       
\usepackage{amsthm}         

\usepackage[capitalize,noabbrev]{cleveref}

\theoremstyle{plain}
\newtheorem{theorem}{Theorem}[section]
\newtheorem{proposition}[theorem]{Proposition}

\theoremstyle{definition}

\newtheorem{assumption}[theorem]{Assumption}
\theoremstyle{remark}

\usepackage{mdframed}

\usepackage{pifont}
\newcommand{\xmark}{\ding{55}} 

\newtheorem{innercustomproposition}{Proposition}
\newenvironment{customproposition}[1]
  {\renewcommand\theinnercustomproposition{#1}\innercustomproposition}
  {\endinnercustomproposition}

\usepackage{enumitem} 

\usepackage[capitalize,noabbrev]{cleveref} 

\usepackage{tikz}  
\usetikzlibrary{arrows.meta,calc,angles,quotes,matrix,positioning,shapes.geometric}
\tikzset{>=Stealth}

\usepackage{graphicx}
\usepackage{ragged2e}
\definecolor{ForestGreen}{RGB}{34,139,34}

\DeclareRobustCommand{\circledgreen}[1]{%
  \tikz[baseline=(c.base)]\node (c)
    [circle,
     draw=ForestGreen,
     fill=ForestGreen!15,
     inner sep=1.2pt]{\footnotesize #1};}

\usepackage[most]{tcolorbox}
\definecolor{navyblue}{RGB}{0,45,114}

\newtcolorbox{remarkbox}[1][]{%
  colback=navyblue!5,    
  colframe=navyblue,     
  boxrule=0.5pt,         %
  arc=2mm,               %
  left=1mm,right=1mm,top=1mm,bottom=1mm, %
  title=\textbf{#1}      %
}

\definecolor{mColor1}{rgb}{0.9,0.9,0.9}
\definecolor{lightblue}{rgb}{0.81, 0.94, 1.0}

\usepackage[textsize=tiny]{todonotes}

\icmltitlerunning{Towards Uniformity and Alignment for Multimodal Representation Learning}

\begin{document}

\twocolumn[
  \icmltitle{Towards Uniformity and Alignment for Multimodal Representation Learning}



  \icmlsetsymbol{equal}{*}

  \begin{icmlauthorlist}
\icmlauthor{Wenzhe Yin}{uva}
\icmlauthor{Pan Zhou}{smu}
\icmlauthor{Zehao Xiao}{uva}
\icmlauthor{Jie Liu}{uva}
\icmlauthor{Shujian Yu}{vu,uit}
\icmlauthor{Jan-Jakob Sonke}{nki}
\icmlauthor{Efstratios Gavves}{uva}
  \end{icmlauthorlist}

\icmlaffiliation{uva}{University of Amsterdam}
\icmlaffiliation{nki}{The Netherlands Cancer Institute}
\icmlaffiliation{vu}{Vrije Universiteit Amsterdam}
\icmlaffiliation{uit}{The Arctic University of Norway}
\icmlaffiliation{smu}{Singapore Management University}

\icmlcorrespondingauthor{Pan Zhou}{panzhou@smu.edu.sg}
  \vskip 0.3in
]



\printAffiliationsAndNotice{}  

\begin{abstract}
Multimodal representation learning aims to construct a shared embedding space in which heterogeneous modalities are semantically aligned. Despite strong empirical results, InfoNCE-based objectives introduce inherent conflicts that yield \emph{distribution gaps} across modalities.
In this work, we identify two conflicts in the multimodal regime, both exacerbated as the number of modalities increases: (i) an \emph{alignment–uniformity} conflict, whereby the repulsion of uniformity undermines pairwise alignment, and (ii) an \emph{intra-alignment} conflict, where aligning multiple modalities induces competing alignment directions. To address these issues, we propose a principled decoupling of alignment and uniformity for multimodal representations, providing a conflict-free recipe for multimodal learning that simultaneously supports discriminative and generative use cases without task-specific modules. We then provide a theoretical guarantee that our method acts as an efficient proxy for a global Hölder divergence over multiple modality distributions, and thus reduces the distribution gap among modalities. Extensive experiments on retrieval and UnCLIP-style generation demonstrate consistent gains. 





\end{abstract}

\section{Introduction}


Multimodal representation learning~\citep{ruan2023accommodating,girdhar2023imagebind} aims to construct a shared embedding space where semantically related signals from different modalities (e.g., image, text, audio, video, speech) are well aligned. A landmark example is CLIP~\citep{radford2021learning}, which employs an InfoNCE objective to align paired image–text representations by maximizing similarity for positive pairs while pushing negative pairs apart. This framework has since been extended beyond two modalities. For instance, ImageBind~\citep{girdhar2023imagebind}, VAST~\citep{chen2023vast}, and LanguageBind~\citep{zhu2023languagebind} incorporate additional
{modalities (e.g., depth and audio)} into a common space, while GRAM~\citep{cicchetti2024gramian} generalizes the cosine similarity to the Gramian volume among multiple modalities in InfoNCE and achieves promising performance.

Despite notable successes, InfoNCE-based methods exhibit inherent conflicts that induce distribution gaps~\citep{liang2022mind,shi2023towards,yin2025distributional}. Consequently, UnCLIP-type generative models (e.g., DALL-E 2~\citep{ramesh2022hierarchical} and Kandinsky~\citep{razzhigaev2023kandinsky}) add a diffusion module to transform CLIP embeddings. Prior work~\citep{yin2025distributional} shows that, in vision–language learning, this gap arises from a conflict between \emph{uniformity} and \emph{alignment}~\citep{wang2020understanding}: the uniformity term spreads embeddings on the unit hypersphere, whereas the alignment term pulls positive (multimodal) pairs together. 
However, in the multimodal regime, it remains unclear {how} these conflicts evolve as the number of modalities \(M\) increases, {which} is crucial for balancing uniformity (discriminability for retrieval) and alignment (closing cross-modal distribution gaps for generation). This is challenging because each modality is jointly determined by multiple cross-modal interactions, making the conflicts hard not only to characterize but also to resolve without sacrificing either retrieval discriminability or generative alignment. To tackle the issue, we explicitly quantify these internal conflicts, which explains how distribution gaps worsen as $M$ grows and directly guides the design of a conflict-free objective.

\textbf{Contributions.} In this work, we systematically analyze and address the inherent conflicts in multimodal InfoNCE that give rise to modality and distributional gaps. Our contributions are four-fold and are highlighted below.

{First, we provide a theoretical analysis of the InfoNCE in multimodal settings (\(M \geq 3\)), which is non-trivial due to a more complex and heterogeneous representation geometry.} We theoretically formalize two distinct conflicts: (1) an alignment-uniformity conflict (\(\zeta_a\)), where uniformity forces oppose alignment, exacerbating distributional gaps across modalities (see Fig.~\ref{fig:conflict-illustration}~(a) and Proposition~\ref{prop:alignment-uniformity-conflict} in Sec.~\ref{theoreticalanalysis}), and (2) an intra-alignment conflict (\(\chi_a\)), driven by non-collinear positive embeddings across modalities, which widens the modality gap as the number  \(M\)  of modalities increases (see Fig.~\ref{fig:conflict-illustration}~(b) and Proposition~\ref{prop:intra-alignmen-conflict} in Sec.~\ref{theoreticalanalysis}). Together, these conflicts explain why multimodal InfoNCE struggles to scale: the same objective that enforces global uniformity undermines the alignment of positive pairs, especially as the modality count $M$ grows.

{Second, to resolve these issues, we propose \textit{UniAlign}, which provides a principled decoupling of \textbf{uni}formity and \textbf{align}ment.} Specifically, we enforce \emph{intra-modality uniformity} within each modality’s samples, ensuring uniform coverage on the unit hypersphere and preventing representation collapse.
In parallel, we introduce an \emph{anchor-based alignment} strategy that aligns embeddings of the same sample across modalities with respect to an anchor modality.
{This explicitly avoids competing alignment directions}, thereby closing modality gaps without introducing competing forces. 
Therefore, UniAlign enhances both cross-modal discriminative and generative capability.



Third, beyond this geometric intuition, we provide a theoretical justification that our method minimizes the distribution gap. Specifically, we introduce a global Hölder divergence applicable to an arbitrary number of modality distributions. We then connect our decoupled losses to this divergence, showing that the intra-modality uniformity and anchor-based alignment terms act as efficient computational proxies for minimizing it, thereby providing formal theoretical justification.


Extensive experimental results demonstrate the consistent superiority of our UniAlign over InfoNCE-based baselines in representation quality, retrieval accuracy, and generation fidelity. Without additional task-specific modules, the learned embeddings support both discriminative (cross-modal retrieval) and generative (UnCLIP-style conditional generation) tasks, yielding around \(2\) R@1 gains and \(10\text{--}40\) lower FID, respectively. These results confirm that our decoupled principle not only resolves the modality and distributional gaps introduced by InfoNCE, but also provides a scalable recipe for robust and versatile multimodal learning. 

\section{Multimodal Conflict Analysis}
\label{theoreticalanalysis}
\begin{figure*}[t!]
  \centering
  \includegraphics[width=0.9\linewidth]{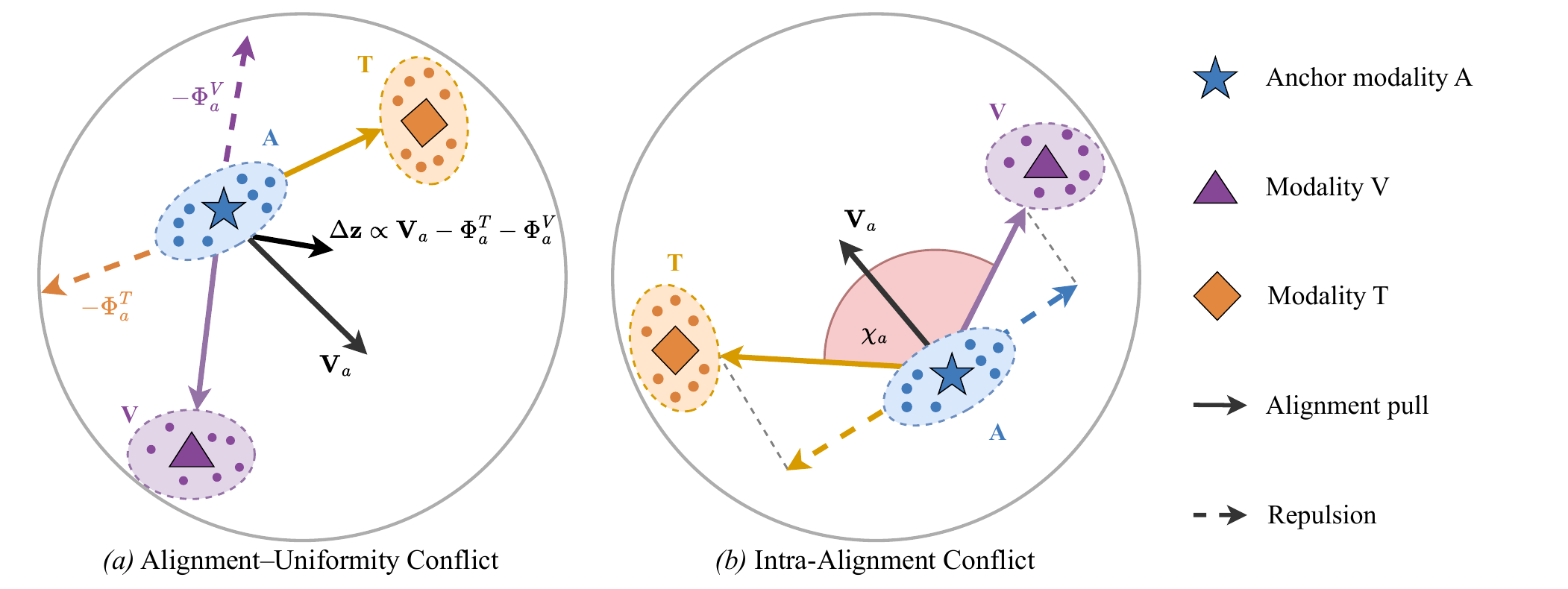}
  \caption{\textbf{Two conflicts of multi-modal InfoNCE.}
  (a) \emph{Alignment--uniformity:} positives are pulled together yet repelled by the uniformity force; 
 (b) \emph{Intra-alignment:} non-collinear positives induce angular tension. Both grow with \(M\).}\label{fig:conflict-illustration}
\vskip -4mm
\end{figure*}

In this section, we first revisit prior analyses of the InfoNCE objective for two modalities (vision and language). {We then quantify two types of conflicts in multimodal learning and present a principled framework that characterizes how conflicts evolve as the number of modalities increases.}

\subsection{Uniformity and Alignment Conflict of InfoNCE.}
 Let \(M\) be the number of modalities and \(B\) the batch size. For sample index \(i\in\{1,\dots,B\}\) and modality \(m\in\{1,\dots,M\}\), denote the \(\ell_2\)-normalized embedding by \( \mathbf{Z}^{(m)} =\{\mathbf{z}_i^{(m)} \in\mathbb{R}^d \}_{i=1}^B\). The generalized multi-modal InfoNCE objective~\citep{oord2018representation} (over all pairs) is:
\begin{equation}
\label{eq:infonce}
\begin{aligned}
&\mathcal{L}_{\text{InfoNCE}} 
=  -\frac{1}{\sum_{m\ne n}w_{mn}}
\sum_{i=1}^B\sum_{m\ne n} w_{mn}\, \log p_{ii}^{(mn)} \\
&\qquad\qquad \text{with }\  p_{ik}^{(mn)}=\frac{\exp({\mathbf{z}_i^{(m)}}^\top \mathbf{z}_k^{(n)}/\tau)}{\sum_{\ell=1}^{B}\exp({\mathbf{z}_i^{(m)}}^\top \mathbf{z}_\ell^{(n)}/\tau)},
\end{aligned}
\end{equation}
where \(w_{mn}>0\) denotes weight and \(\tau>0\) is the temperature. This loss has been extensively used in recent multimodal applications~\citep{girdhar2023imagebind,guzhov2022audioclip}. Then for two modalities,  InfoNCE can be decomposed into \emph{alignment} and \emph{uniformity} \citep{wang2020understanding}:
\begin{equation}
\label{eq:uniform-align}
\begin{aligned}
&\text{Alignment:}\ \ \mathbb{E}_{p_{\text{pair}}}\!\big[\|\mathbf{z}^{(1)}-\mathbf{z}^{(2)}\|_2^2\big],
\\
&\text{Uniformity:}\ \ \log\mathbb{E}_{p_{\text{data}}}\!\big[\exp(-\|\mathbf{z}^{(1)}-\mathbf{z}^{(2)}\|_2^2/2\tau)\big],
\end{aligned}
\end{equation}
{where \(p_{\text{pair}}\) is defined on cross-modal pairs, and $p_{\text{data}}$ is the overall data distribution regardless of pairing.}
Uniformity spreads embeddings over the unit hypersphere, thereby avoiding collapse and promoting semantic coverage, while alignment pulls paired cross-modal representations together to enforce semantic consistency. In vision–language learning, \citet{yin2025distributional} demonstrates that uniformity \emph{across} modalities (``inter-uniformity'') conflicts with the alignment term, resulting in a systemic distributional gap (see evidence in Appendix~\ref{app:evidence-conflict}).

However, when extending to more modalities, the analysis is insufficient to present the relationship between conflict degree and the number of modalities, which is important to understand the learning mechanism of multimodal representation. Due to the more complex representation space geometry of multiple modalities, where each modality is influenced by multiple factors, quantifying the conflict in multimodal representation learning is challenging.

\subsection{Systematic Multimodal Conflict Analysis}


We first reveal two modes of conflict in multimodal learning with InfoNCE, and then prove that the two conflicts become severe when the number of modalities $M$ increases by Proposition~\ref{prop:alignment-uniformity-conflict} and~\ref{prop:intra-alignmen-conflict}. 

To quantify the conflict in multimodal learning, we first choose one modality $\mathbf{z}^{(a)}$ as the anchor, and analyze how it is
influenced by other modalities from the gradient perspective. 
Differentiating Eq.~(\ref{eq:infonce}) with respect to an anchor \(\mathbf{z}_i^{(a)}\) exposes a ``push–pull'' structure. For a modality pair \((a\!\to\! n)\),
\begin{equation}
\label{eq:grad-vphi}
\begin{aligned}
& \nabla_{\mathbf{z}_i^{(a)}} \mathcal{L}
= - \underbrace{\sum_{n\neq a}\frac{w_{an}}{\tau}\,\mathbf{z}_i^{(n)}}_{\mathbf{V}_a}
+ \underbrace{\sum_{n\neq a}\frac{w_{an}}{\tau}\sum_{k=1}^{B} p_{ik}^{(an)}\,\mathbf{z}_k^{(n)}}_{\bm{\Phi}_a},
\end{aligned}
\end{equation}
where the first term $\mathbf{V}_a$
aggregates only the matched positives across modalities and thus acts as an attraction (alignment) force, and the second term
$\bm{\Phi}_a$
is a softmax-weighted mixture over the batch and captures the repulsion induced by negatives, i.e., the uniformity-related force.

Eq.~(\ref{eq:grad-vphi}) exposes two levels of conflicts. 
(i) \emph{Inter-modality alignment–uniformity conflict}: when the uniformity push and alignment pull are directionally aligned, i.e., $\langle \mathbf{V}_a,\mathbf{\Phi}_a\rangle>0$, the term $-\mathbf{V}_a+\mathbf{\Phi}_a$ cancels in the gradient, yielding a small update for $\mathbf{z}^{(a)}$ (see Fig.~\ref{fig:conflict-illustration} (a)).
{This conflict is rooted in \emph{inter-modality} uniformity, which induces repulsion across modalities and leads to a distribution gap.} 
{(ii) \emph{Intra-alignment conflict}: the net alignment force weakens when the positive targets are not collinear (Fig.~\ref{fig:conflict-illustration}(b)). Since \(\mathbf{V}_a=\sum_{n\neq a}\frac{w_{an}}{\tau}\mathbf{z}^{(n)}\) aggregates these positive pulls via a vector sum, collinear positives reinforce each other and yield a large \(\|\mathbf{V}_a\|\), whereas non-collinear or opposing positives partially cancel, reducing the alignment signal.}



\textbf{Conflict quantification (\(\zeta_a,\chi_a\)).} 
To quantitatively analyze these two conflicts, we define the \emph{alignment-uniformity conflict} \(\zeta_a\in[-1,1]\) to measure directional opposition between the alignment and uniformity forces, and introduce  the intra-alignment conflict \(\chi_a\in[0,1]\) to quantify cancellation among non-collinear positive pulls within \(\mathbf{V}_a\) :
\begin{equation}
\begin{aligned}
	& \zeta_a \triangleq \cos\!\big(\mathbf{V}_a,\bm{\Phi}_a\big)={\mathbf{V}_a^\top \bm{\Phi}_a}/{\big(\|\mathbf{V}_a\|_2\,\|\bm{\Phi}_a\|_2\big)}.
	\\
	&\chi_a \triangleq 1-{\|\mathbf{V}_a\|_2}/{\big(\sum\nolimits_{n\neq a} w_{an}/\tau\big)},
\end{aligned}
\end{equation}
A high positive $\zeta_a$ (near 1) indicates severe conflict, which occurs when hard negatives lie in the same direction as positives. $\chi_a$ {measures how much the positive ``pull'' vectors within $\mathbf{V}_a$ cancel due to non-collinearity}. A value of $\chi_a=0$ indicates perfect alignment (no conflict), whereas $\chi_a \to 1$ indicates severe conflict. 

To study how the alignment–uniformity conflict scales with the number of modalities M, we next introduce a mild structural assumption on each per-modality uniformity component. Intuitively, it separates a systematic component that is coherent across modalities from residual variations that are largely modality-specific.

\begin{assumption}[Systematic conflict per-modality ]\label{assump:sys-conflict}
Let $\hat{\mathbf{V}}_a = \mathbf{V}_a/\|\mathbf{V}_a\|$ denote the unit alignment direction for anchor $a$.
For each modality $n\neq a$, the uniformity component admits the decomposition
\begin{equation}
\bm{\Phi}_a^{(n)} \;=\; c_n\,\hat{\mathbf{V}}_a \;+\; \bm{\varepsilon}_n,
\end{equation}
where
$c_n \triangleq \langle \bm{\Phi}_a^{(n)}, \hat{\mathbf{V}}_a\rangle$ quantifies the magnitude of systematic conflict from modality $n$ in this direction, and satisfies $c_n \ge c_0 $ for positive constant $c_0$. The residuals $\bm{\varepsilon}_n$ are zero-mean, mutually independent, and have bounded covariance.
\end{assumption}

\emph{Assumption Justification.} Most negatives in large-batch multimodal training are semantically unrelated and approximately isotropic, so their contributions largely average out as zero-mean noise. A small fraction of hard negatives, however, concentrates near the positive alignment direction and yields a consistent projection onto \(\hat{\mathbf{V}}_a\). This motivates decomposing each term into a systematic component \(c_n\) (hard negatives) plus a residual \(\bm{\varepsilon}_n\) (easy negatives and noise), and implies that as \(M\) increases, the systematic parts add coherently while residuals average out (see Appendix~\ref{app:alignment-uni-conflict}).


\begin{proposition}[Alignment–Uniformity Conflict]\label{prop:alignment-uniformity-conflict}
Let $\bm{\Phi}_a=\sum_{n\neq a}\bm{\Phi}_a^{(n)}$ be the total uniformity force on anchor $a$, and define
$\zeta_a \!=\! \cos\!\big(\mathbf{V}_a,\bm{\Phi}_a\big)$. 
Under Assumption~\ref{assump:sys-conflict}, the alignment–uniformity conflict converges to its maximum as the number of modalities $M$ increases:
\begin{equation}    
\mathbb{E}[\zeta_a] \;=\; \mathbb{E}\!\left[\cos\!\big(\mathbf{V}_a,\bm{\Phi}_a\big)\right] 
\;\longrightarrow\; 1
\qquad \text{as } M\to\infty .
\end{equation}
\end{proposition}
See its proof in Appendix~\ref{app:alignment-uni-conflict}. This shows that even if the conflict from each modality ($c_n$) is small, their systematic accumulation inevitably causes the total, observable conflict ($\zeta_a$) to become severe.

\begin{proposition}[Intra-alignment Conflict]
\label{prop:intra-alignmen-conflict}
The expected intra-alignment conflict, $\mathbb{E}[\chi_a]$, is governed by $M$ and the average pairwise alignment \(\bar{\mu} = \mathbb{E}[\mathbf{z}_i^{(m)\top} \mathbf{z}_i^{(n)}] \in [0, 1] \) for \(m \neq n\) between modalities:
\begin{equation}
	\label{eq:conflict_scaling_law}
	\mathbb{E}[\chi_a] \ \ge\ 1 - \sqrt{({1+(M-2)\bar{\mu}})/({M-1})}.
\end{equation}
For imperfect alignment (\(\bar{\mu} < 1\)), the conflict increases with the number of modalities \(M\) and admits a non-zero asymptotic lower bound:
\begin{equation}
\liminf\nolimits_{M \to \infty} \mathbb{E}[\chi_a] \ \ge\ 1 - \sqrt{\bar{\mu}}.
\end{equation}
\end{proposition}

See its proof in Appendix~\ref{app:proof-intra-alignment-coro}. Proposition~\ref{prop:intra-alignmen-conflict} shows that the internal conflict of the alignment force gets severe with more modalities, resulting in ineffective alignment. 

 By combining Proposition~\ref{prop:alignment-uniformity-conflict} and~\ref{prop:intra-alignmen-conflict}, one can conclude that the standard multi-modal InfoNCE objective is fraught with a two-level conflict system: an intra-alignment conflict, and a classic alignment-uniformity conflict. Such conflicts result in a distinct distributional modality gap. This motivates the exploration of alternative frameworks that \textbf{decouple} these objectives, by optimizing uniformity separately and employing a more direct, conflict-free alignment mechanism.

\section{Methodology}


In Section~\ref {theoreticalanalysis}, our  analysis has identified two fundamental  conflicts that impede multi-modal contrastive learning: the intra-alignment conflict ($\chi$), and the alignment-uniformity conflict ($\zeta$). To circumvent these issues, we propose a generic principle to  decouple the learning objectives. Then, we show that our principle essentially minimizes the global distribution gap with a theoretical guarantee.

\subsection{General Principle for Multimodal Learning}
\label{sec:method-principle}

\begin{remarkbox}[A general principle for multimodal learning]
As the \textit{alignment-uniformity} and intra-alignment conflicts are the root for the modality/distribution gap, avoiding these conflicts from the uniformity and alignment perspective is a good general principle: 
\begin{enumerate}[label=\circledgreen{\arabic*}]
\item \textbf{Intra-modality uniformity only:}
encourage uniform spread \emph{within each modality} $\{\mathbf{z}_i^{(m)}\}$ on $\mathbb{S}^{d-1}$, and avoid any
\emph{cross-modality} uniformity/repulsion.
\item \textbf{Conflict-free alignment:} Explicitly or implicitly maximize the consensus magnitude to avoid the non-collinearity problem between positive pairs. 
\end{enumerate}
\circledgreen{1} avoids the cross-modality uniformity conflict but still pushes the embeddings uniformly spreading in a unit hypersphere. \circledgreen{2} avoids the non-collinear positive pulls in the consensus vector.
\end{remarkbox}


Following this generic principle, we present our design for the uniformity and alignment terms in Euclidean space. See a summarization in  Table~\ref{table:principle} of Appendix~\ref{app:design-space}.

\noindent{\textbf{Uniformity.}}
Let \(\mathbf{Z}^{(m)}=\{\mathbf{z}_i^{(m)} \in \mathbb{R}^d \}_{i=1}^B\) denote a batch of unit-normalized embeddings from modality \(m\). To promote uniformity of multimodal representations and mitigate inter-modality conflict, we adopt an \emph{intra-modality} uniformity term to prevent collapse:
\begin{equation}
\label{eq:intra-uniform}
\begin{aligned}
&U\!\big(\mathbf{Z}^{(m)}\big)
= \frac{1}{B}\sum_{i=1}^B
\log\!\left(\frac{1}{B-1}\sum\nolimits_{j\neq i}
\kappa\!\big(\mathbf{z}_i^{(m)}, \mathbf{z}_j^{(m)}\big)\right),
\\
&\qquad\qquad \text{with}\quad  \kappa(\mathbf{z}_i,\mathbf{z}_j)=\exp\!\left(-{\|\mathbf{z}_i-\mathbf{z}_j\|_2^2}/{2\tau^2}\right),
\end{aligned}
\end{equation}
where \(\tau>0\) is the temperature and \(\kappa\) is a Gaussian kernel.
The sample-wise gradient satisfies
$
\nabla_{\mathbf{z}_i^{(m)}} U
= -\frac{1}{\tau^2}\sum_{j\ne i} p_{ij}\,\big(\mathbf{z}_i^{(m)}-\mathbf{z}_j^{(m)}\big),
\quad
p_{ij}=\frac{\exp\!\big(-\|\mathbf{z}_i^{(m)}-\mathbf{z}_j^{(m)}\|_2^2/(2\tau^2)\big)}{\sum_{\ell\ne i}\exp\!\big(-\|\mathbf{z}_i^{(m)}-\mathbf{z}_\ell^{(m)}\|_2^2/(2\tau^2)\big)},
$
so the softmax weights \(p_{ij}\) decay exponentially with distance, effectively suppressing far-away contributions.
Consequently, the uniformity term concentrates gradients on nearby hard negatives (semantically similar samples) while leaving distant points largely unaffected, improving local uniformity and preventing collapse. Note that, as the $U(\mathbf{Z}^{(m)})$ is defined for each modality separately, which is different from the $\Phi_a$ term in Eq.~\eqref{eq:infonce}, hence, the conflict $\zeta_a$ is avoided.

\noindent{\textbf{{Conflict-free alignment.}}
To address the intra-alignment conflict (\(\chi\)) inherent in standard multimodal (\(M\!\ge\!3\)) contrastive objectives, we propose an anchor-based alignment. This design enforces a single positive-pull direction per sample, thereby avoiding non-collinearity among positive pairs.

Specifically, we choose one modality \(\mathbf{Z}^{(a)}\) as the anchor, which serves as a reference direction. As in other modalities, we apply intra-modality uniformity to the anchor embeddings, while aligning all remaining modalities to this anchor. As a result, each sample is aligned along a single direction, preventing competing (non-collinear) positive pulls.

Hence, the most straightforward alignment loss on the Euclidean space can be defined by:
\begin{equation}
\label{eq:alignment}
L_{\mathrm{align}} = \frac{1}{B(M-1)} \sum\nolimits_{i=1}^B \sum\nolimits_{n\neq a}
\bigl\lVert \mathbf{z}_i^{(a)} - \mathbf{z}_i^{(n)} \bigr\rVert_2^2 .
\end{equation}


\noindent{\textbf{Overall objective.}}
The final objective with hyperparameter $\lambda_{\mathrm{align}} > 0$ is defined by:
\begin{equation}
\label{eq:overall-loss}
\mathcal{L}
=  \sum\nolimits_{m=1}^{M} U\!\big(Z^{(m)}\big)
+ \lambda_{\mathrm{align}}\, L_{\mathrm{align}}.
\end{equation}

\noindent{\textbf{Hypersphere space.}} 
Our principle is not restricted to Euclidean space and can be extended to manifolds.
Since both InfoNCE and our uniformity term encourage representations to spread on the unit hypersphere, a natural alternative is to measure similarity via the geodesic distance on $\mathbb{S}^{d-1}$:
\begin{equation}
\begin{aligned}
&d_{\mathbb{S}}(\mathbf{z}^{i},\mathbf{z}^{j}) \;=\; \arccos\!\big(\langle \mathbf{z}^{i},\mathbf{z}^{j} \rangle\big), \qquad \|\mathbf{z}^{i}\|_2=\|\mathbf{z}^{j}\|_2=1, \\
&k_{\mathbb{S}}(\mathbf{z}^{i},\mathbf{z}^{j};\tau) \;=\; \exp\!\Bigl(-\,{d_{\mathbb{S}}(\mathbf{z}^{i},\mathbf{z}^{j})^2}/{2\,\tau^2}\Bigr).
\end{aligned}
\end{equation}
This yields a drop-in hyperspherical variant within the same framework, illustrating the flexibility of our design.
We validate this choice in Table~\ref{table:generation} and summarize Euclidean vs.\ manifold instantiations in Table~\ref{table:principle} (Appendix~\ref{app:design-space}).



\subsection{Theoretical Analysis from Divergence Perspective}
In the previous section, we introduced our objective based on the proposed principle. A natural question is whether this objective is theoretically guaranteed to reduce the distribution gap across modalities. Here, we show that optimizing intra-modality uniformity and cross-modality alignment minimizes a global distribution divergence, thereby mitigating the cross-modal (distribution) gap.

Classical divergences~\citep{jenssen2006cauchy,shlens2014notes} are typically defined between two distributions, but our setting involves $M$ modalities. We therefore introduce a new \emph{global Hölder divergence} to jointly measure the discrepancy among all modality distributions. Let $\{p_{m}(\mathbf{z})\}_{m=1}^M$ denote the densities of the $M$ modalities. By Hölder’s inequality, they satisfy
\begin{equation}
	\left| \int \prod\nolimits_{m=1}^M p_m (\mathbf{z}) \, d\mathbf{z} \right|^M \leq \prod\nolimits_{m=1}^M \int |p_{m}(\mathbf{z})|^M d\mathbf{z},
\label{eq:div-def}
\end{equation}
and takes equality if and only if \(p_1=\cdots=p_M\). This inequality motivates the definition of the global Hölder divergence as the log of the ratio between the two sides of Eq.~\eqref{eq:div-def}:
\begin{equation}
\begin{aligned}
&D_{\text{Hölder}}
= -\log \frac{\int \prod_{m=1}^M p_{m}(\mathbf{z})\, d\mathbf{z}}
{\left(\prod_{m=1}^M \int |p_{m}(\mathbf{z})|^M d\mathbf{z} \right)^{\frac{1}{M}}} \\ 
&= \underbrace{\frac{1}{M}\sum_{m=1}^{M} \log \int |p_{m}(\mathbf{z})|^M d\mathbf{z}}_{\text{Uniformity Term}} \underbrace{ -\log \int \prod_{m=1}^M p_{m}(\mathbf{z})\, d\mathbf{z}}_{\text{Alignment Term}}.
\end{aligned}
\label{eq:divergence}
\end{equation}



We empirically estimate this global divergence in a non-parametric way via a kernel density estimator (KDE) with a Gaussian kernel. Under the KDE plug-in estimator derived in Appendix~\ref{app:divergence}, the term $\frac{1}{M}\sum_{m=1}^{M} \log \int |p_{m}(\mathbf{z})|^M d\mathbf{z}$ can be approximated by averaging Gaussian similarities between samples within the same modality, which motivates our intra-modality uniformity objective $U(Z^{(m)})$ in Eq.~\eqref{eq:intra-uniform} (up to constants). Likewise, the second term $-\log \int \prod_{m=1}^M p_{m}(\mathbf{z})\, d\mathbf{z}$ measures how much the modality distributions overlap in the shared embedding space. It can be approximated by averaging Gaussian similarities between embeddings from different modalities. Maximizing this overlap admits an efficient instance-matching surrogate in the low-temperature limit $(\tau \to 0)$, motivating $L_{\mathrm{align}}$ in Eq.~\eqref{eq:alignment}. Therefore, optimizing intra-modality uniformity together with cross-modality alignment provides a tractable way to reduce the global Hölder divergence in Eq.~\eqref{eq:divergence}.

\subsection{Tuple-Level Extensions}
Our design principle is \emph{instantiation-agnostic}: it targets uniformity across samples and alignment across modalities, while remaining compatible with a broad class of loss constructions.
Beyond anchor-wise (pairwise) formulations, a complementary direction is to treat each multimodal tuple as a \emph{single structured sample} and regularize its \emph{within-tuple geometry}.
This tuple-level view strengthens both uniformity and alignment from a different angle.

We propose a tuple-level complement that (i) enforces \emph{uniform dispersion} of tuples in the representation space, and (ii) encourages \emph{cross-modal collinearity} within each tuple.

\noindent\textbf{Tuple-level uniformity \(U(\mathbf{C})\). }
Given a multimodal tuple \(\{\mathbf{z}_i^{(1)},\ldots,\mathbf{z}_i^{(M)}\}\), we represent it by a weighted and normalized centroid \(\mathbf{c}_i\), and denote \(\mathbf{C}=\{\mathbf{c}_i\}_{i=1}^{B}\):
\begin{equation}
\label{eq:centroid-def}
\begin{aligned}
	&\mathbf{c}_i
	= {\sum\nolimits_{m=1}^{M} w_m\, \mathbf{z}_i^{(m)}}/{\big\|\sum\nolimits_{m=1}^{M} w_m\, \mathbf{z}_i^{(m)}\big\|_2},
	\\
	&\qquad \text{with} \quad w_m \ge 0,\quad \sum\nolimits_{m=1}^{M} w_m = 1.
\end{aligned}
\end{equation}
We then apply the same uniformity objective to \(\mathbf{C}\), i.e., \(U(\mathbf{C})\) (default \(w_m=1/M\)), which encourages tuples (rather than individual modalities) to spread out and improves global separability.

\noindent\textbf{Tuple-level alignment $L_{\mathrm{vol}}$.}
To capture agreement among all modalities within each tuple, we penalize the \emph{Gram-determinant volume} spanned by \(\{\mathbf{z}_i^{(m)}\}_{m=1}^{M}\), which equals zero when the modality embeddings are collinear.
Let \(\mathbf{G}_i \in \mathbb{R}^{M\times M}\) be the Gram matrix~\citep{cicchetti2024gramian} with \([\mathbf{G}_i]_{mn}=\langle \mathbf{z}_i^{(m)}, \mathbf{z}_i^{(n)}\rangle\).
The induced volume is proportional to \(\sqrt{\det(\mathbf{G}_i)}\), and we define
\begin{equation}
\label{eq:vol-loss}
L_{\mathrm{vol}} \;=\; \frac{1}{B}\sum_{i=1}^{B}\sqrt{\det(\mathbf{G}_i)}.
\end{equation}
Minimizing \(L_{\mathrm{vol}}\) explicitly promotes cross-modal collinearity, complementing the anchor-based alignment.

\section{Related Work}

\begin{table*}[t!]
\centering
\caption{\textbf{Zero‐shot multimodal text‐to‐video (T2V) and video‐to‐text (V2T) retrieval results (Recall@1).} Our method, UniAlign, consistently outperforms baselines in most tasks. * denotes the results without tuple-level losses ($U(\mathbf{C})$ and $L_{\mathrm{vol}}$).}
\label{tab:zero_shot}
\resizebox{0.8\textwidth}{!}{%
  \begin{tabular}{@{} l c  cc  cc  cc  cc @{}}
    \toprule
    Method & Modality
      & \multicolumn{2}{c}{MSR‐VTT}
      & \multicolumn{2}{c}{DiDeMo}
      & \multicolumn{2}{c}{ActivityNet}
      & \multicolumn{2}{c}{Average} \\
    \cmidrule(lr){3-4}\cmidrule(lr){5-6}\cmidrule(lr){7-8}\cmidrule(lr){9-10}
           &         & T2V & V2T & T2V & V2T & T2V & V2T & T2V & V2T \\
    \midrule
    UMT~\citep{liu2022umt}                  & T–V  & 33.3 & --  & 34.0 & --   & 31.9 & --   & 33.1 & -- \\
    OmniVL~\citep{wang2022omnivl}           & T–V  & 34.6 & --  & 33.3 & --   & --   & --   & 34.0 & -- \\
    UMT‐L~\citep{li2023unmasked}            & T–V  & 40.7 & 37.1& 48.6 & 49.9 & 41.9 & 39.4 & 43.7 & 42.1 \\
    TVTSv2~\citep{zeng2023tvtsv2}           & T–V  & 38.2 & --  & 34.6 & --   & --   & --   & 36.4 & -- \\
    ViCLIP~\citep{wang2023internvid}        & T–V  & 42.4 & 41.3& 18.4 & 27.9 & 15.1 & 24.0 & 25.3 & 31.1 \\
    VideoCoCa~\citep{yan2022videococa}      & T–V  & 34.3 & 64.7& --   & --   & 34.5 & 33.0 & 34.4 & 48.9 \\
    Norton~\citep{lin2024multi}             & T–V  & 10.7 & --  & --   & --   & --   & --   & 10.7 & -- \\
    ImageBind~\citep{girdhar2023imagebind}  & T–V  & 36.8 & --  & --   & --   & --   & --   & 36.8 & -- \\
    InternVideo‐L~\citep{wang2022internvideo}& T–V & 40.7 & 39.6& 31.5 & 33.5 & 30.7 & 31.4 & 34.3 & 34.8 \\
    HiTeA~\citep{ye2023hitea}               & T–V  & 34.4 & --  & 43.2 & --   & --   & --   & 38.8 & -- \\
    mPLUG‐2~\citep{xu2023mplug}             & T–V  & 47.1 & --  & 45.7 & --   & --   & --   & 46.4 & -- \\
    VideoPrism‐b~\citep{zhao2024videoprism} & T–V  & 51.4 & 50.2& --   & --   & 49.6 & 47.9 & 50.5 & 49.1 \\
    LanguageBind~\citep{zhu2023languagebind}& T–V  & 44.8 & 40.9& 39.9 & 39.8 & 41.0 & 39.1 & 41.9 & 39.9 \\
    VAST~\citep{chen2023vast}        & T–VA & 49.3 & 43.7& 49.5 & 48.2 & 51.4 & 46.8 & 50.1 & 46.2 \\
    GRAM~\citep{cicchetti2024gramian}                             & T–VA & 54.2 & 50.5& 54.2 & \bf{52.2} & 59.0 & 50.4 & 55.8 & 51.0 \\
    \midrule
    \rowcolor{lightblue}
    UniAlign* (Ours)                                    & T–VA & {57.7} & {53.2} & {55.2} & {51.9} & {59.2} & \bf{52.5} & {57.4} & {52.5} \\
    \rowcolor{lightblue}
    UniAlign (Ours)                                    & T–VA & \bf{58.7} & \bf{54.6} & \bf{58.2} & 51.6 & \bf{59.4} & {51.7} & \bf{58.8} & \bf{52.6} \\
    \bottomrule
  \end{tabular}%
}
\vspace{-4mm}
\end{table*}

CLIP~\citep{radford2021learning} pioneered aligning two modalities (vision and language) using the InfoNCE objective~\citep{oord2018representation}. It has enabled substantial progress in image–text retrieval~\citep{jang2024mate,koukounas2024jina,huang2024llm2clip} and text-to-image (T2I) generation~\citep{ramesh2022hierarchical,rombach2022high}. CLIP-style contrastive objectives have since been applied to additional modality pairs, including audio–text~\citep{elizalde2023clap,wu2023large} and point cloud–text~\citep{zhang2022pointclip}. Beyond pairs, recent work such as CMRC~\citep{wang2023connecting}, CLIP4VLA~\citep{ruan2023accommodating}, ImageBind~\citep{girdhar2023imagebind}, and LanguageBind~\citep{zhu2023languagebind} extends CLIP by introducing more modalities (e.g., video, audio, depth, IMU) into a unified space using pairwise InfoNCE objectives. In parallel, VAST~\citep{chen2023vast}, mPLUG-2~\citep{xu2023mplug}, and InternVideo2~\citep{wang2024internvideo2} advance the state of the art through large-scale training and architectural refinements. Complementing these trends, GRAM~\citep{cicchetti2024gramian} introduces the cross-modality Gram matrix to replace pairwise cosine similarity in InfoNCE with a volume score given by the modality Gram matrix to better handle multimodal alignment.

Despite the success of these methods, embeddings from different modalities still exhibit distinct \emph{distribution gaps} (Fig.~\ref{fig:tsne-vis}), largely attributable to the InfoNCE objective. Prior studies~\citep{zhou2023clip,liang2022mind,shi2023towards} have reported this phenomenon in vision–language learning: \citet{liang2022mind} observe that the InfoNCE objective can encourage modality gaps, while \citet{yin2025distributional} provide a theoretical account showing that uniformity and alignment~\citep{wang2020understanding} conflict, inducing persistent distributional discrepancies. However, these analyses are restricted to the bimodal case; a principled understanding of the conflict mechanisms in the \emph{multimodal} regime remains lacking, partly due to the geometric complexity of shared representation spaces. In this work, we systematically analyze these conflicts for general multimodal learning and, based on this analysis, propose a generic principle for multimodal representation learning.

\section{Experiments}
\label{sec:experiments}

We evaluate UniAlign from two aspects: (i) embedding separability via cross-modal retrieval (Sec.~\ref{sec:exp-retrieval}), and (ii) the distributional modality gap via cross-modal generation with fixed image decoders (Sec.~\ref{sec:exp-generation}). Additional implementation details and results are provided in Appendix~\ref{app:experiments}.

\subsection{Cross-modal Retrieval}
\label{sec:exp-retrieval}


\noindent\textbf{Experimental setting.} 
Following GRAM~\citep{cicchetti2024gramian}, we train on VAST150K~\citep{chen2023vast} and evaluate \emph{zero-shot} video retrieval on three standard benchmarks: MSR-VTT~\citep{xu2016msr}, DiDeMo~\citep{anne2017localizing}, and ActivityNet~\citep{caba2015activitynet}. For fair comparison, we use the same modality encoders as VAST/GRAM: BERT-B for text, BEATs~\citep{chen2022beats} for audio, and EVA-CLIP ViT-G~\citep{sun2023eva} for video. We report zero-shot Recall@1 (R@1) for both text-to-video (T2V) and video-to-text (V2T) {retrieval}. For joint retrieval, where text queries the most similar video-audio tuple (T-VA) and vice versa, we follow GRAM and rank samples by the \emph{volume} score (the determinant of the Gramian matrix).


\begin{figure*}[t!] 
    \centering
    \begin{minipage}{0.32\textwidth}
        \centering
\includegraphics[width=\linewidth]{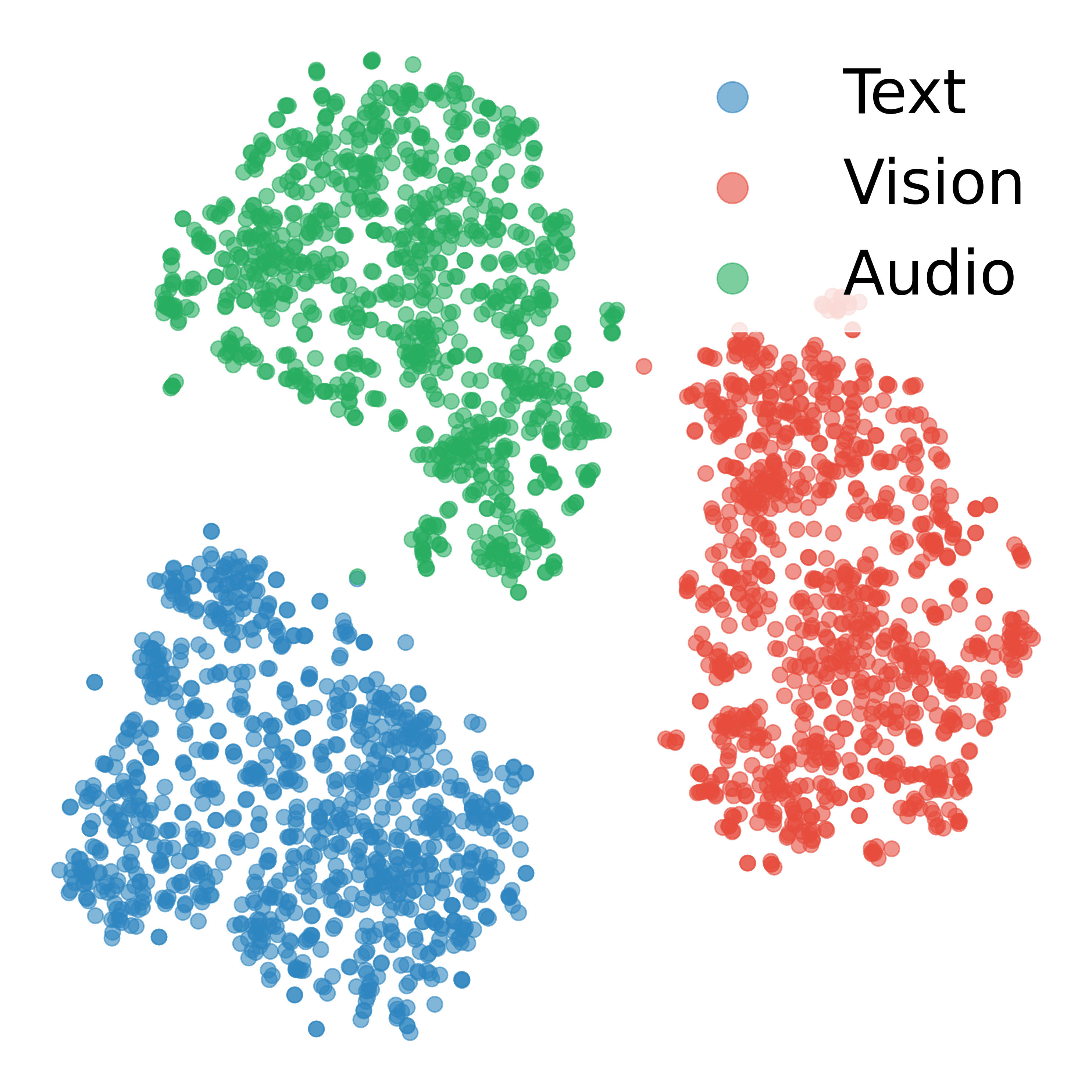}
        \vskip -3.5mm
        \subcaption{ImageBind (paired InfoNCE).}
        \label{fig:tsne-infonce}
    \end{minipage}%
    \hfill
    \begin{minipage}{0.32\textwidth}
        \centering
\includegraphics[width=\linewidth]{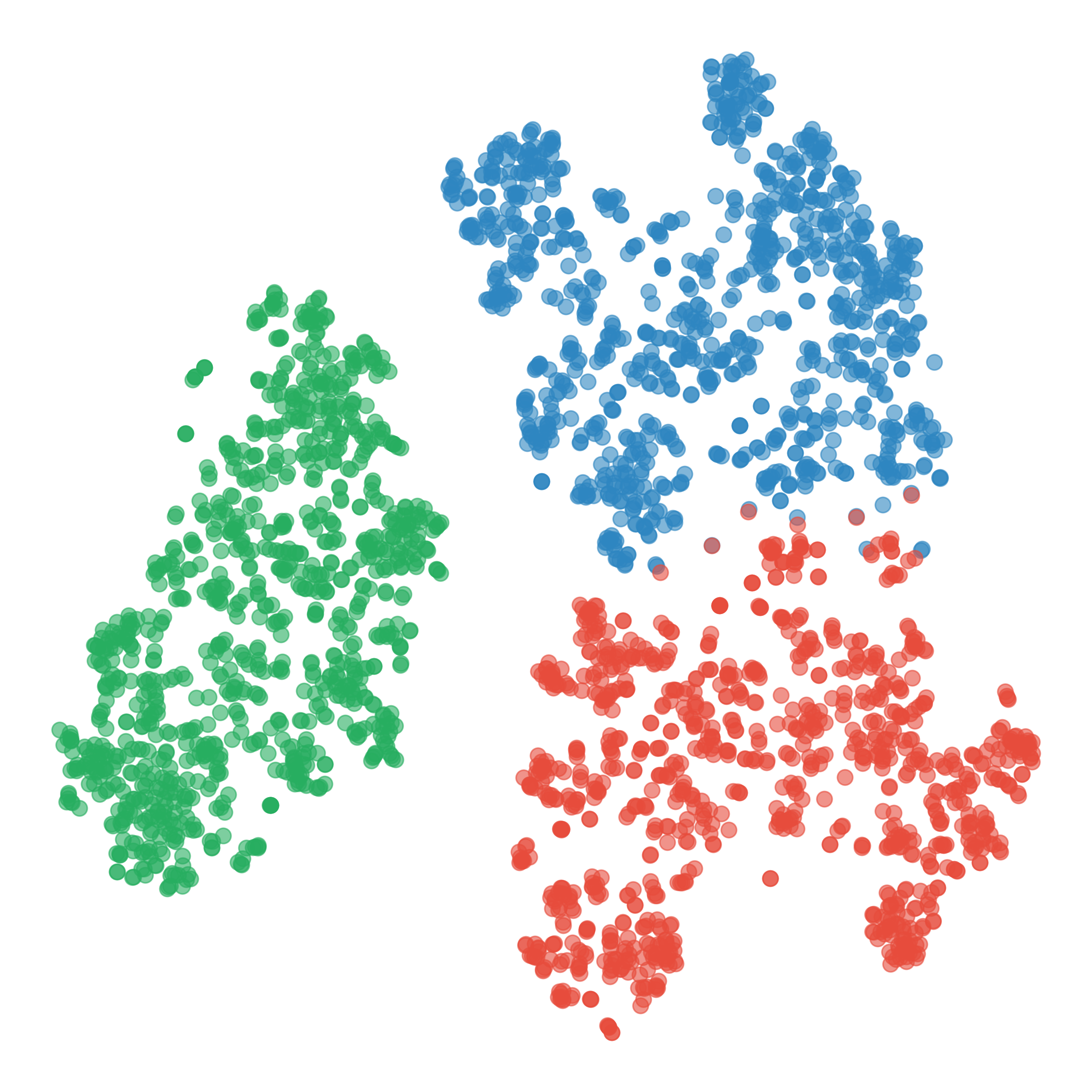}
        \vskip -3.5mm
        \subcaption{GRAM (volume InfoNCE).}
        \label{fig:tsne-gram}
    \end{minipage}%
    \hfill
    \begin{minipage}{0.32\textwidth}
        \centering        \includegraphics[width=\linewidth]{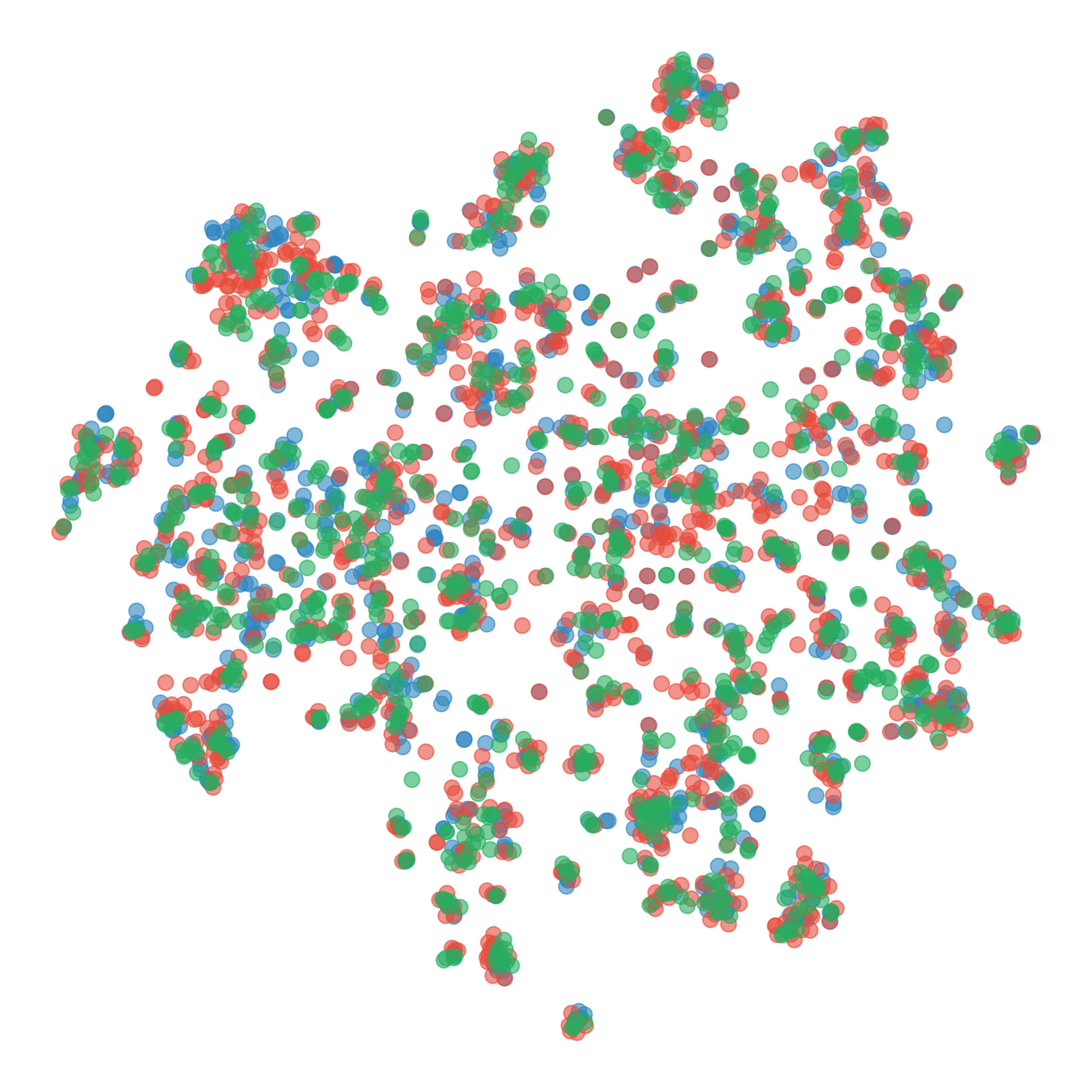}
        \vskip -3.5mm
        \subcaption{Ours}
        \label{fig:tsne-ours}
    \end{minipage}
\vskip -0.5mm
\caption{\textbf{ T-SNE visualizations of vision, text, and audio features.} InfoNCE-type of objective results in clear distribution gaps (\ref{fig:tsne-infonce} and \ref{fig:tsne-gram}). Our method mitigates the distribution gap (\ref{fig:tsne-ours}). }
\vskip -4.5mm
\label{fig:tsne-vis}
\end{figure*}

\noindent\textbf{Experimental results.} 
Table~\ref{tab:zero_shot} shows that UniAlign consistently outperforms strong baselines on zero-shot T2V/V2T retrieval. Notably, UniAlign* already yields solid gains even without tuple-level losses ($U(\mathbf{C})$ and $L_{\mathrm{vol}}$).
Starting from the same VAST pretrained weights, UniAlign improves VAST by \(8.7\) R@1 on T2V and \(6.4\) R@1 on V2T, demonstrating the effectiveness of our principle-instantiated objective.
Moreover, since VAST uses pairwise InfoNCE, it is affected by both conflicts.
GRAM introduces a Gramian volume score inside InfoNCE, which encourages collinearity (the volume is minimized when vectors become collinear) and thus alleviates the {intra-alignment} conflict, but the alignment-uniformity conflict remains.
By explicitly addressing both conflicts, UniAlign further improves over GRAM by \(3.0\) R@1 on T2V and \(1.6\) on V2T.

\begin{table}[t]
\vskip .5mm
\caption{\textbf{Ablation on $U(C)$ and $L_{\mathrm{vol}}$.} Both tuple-level uniformity $U(C)$ and tuple-level alignment $L_{\mathrm{vol}}$ yield consistent gains.}
\centering
\setlength{\tabcolsep}{6pt}
\begin{tabular}{ccccc}
\toprule
$U(C)$ & $L_{\mathrm{vol}}$ & T2V & V2T & Avg. \\
\midrule
\xmark & \xmark &  36.5 & 36.8 & 36.6 \\
\xmark & \cmark &  38.3 & 36.9 & 37.6 \\
\cmark & \xmark & 37.4 & 39.8 & 38.6 \\
\cmark & \cmark & 40.2 & 43.4 & 41.8 \\
\bottomrule
\end{tabular}
\label{table:ablation-retrieval}
\vskip -4.5mm
\end{table}



\noindent{\textbf{Ablation study.}} To understand the effects of the tuple-level uniformity and alignment, $U(C)$ and $L_{\mathrm{vol}}$, we ablate both components on MSR-VTT (training and testing). To isolate the contribution of our objectives, we use plain cosine-similarity retrieval, excluding common post-hoc refinements such as similarity matrix or image-text matching re-ranking. As shown in Table~\ref{table:ablation-retrieval}, both the tuple-level uniformity $U(C)$ and alignment $L_{\mathrm{vol}}$ improve embedding separability and consistently boost retrieval performance. 
More ablations are provided in Appendix~\ref{app:ablation}.


\subsection{Cross-modal Generation}
\label{sec:exp-generation}

\begin{table*}[t]
\caption{\textbf{Cross-modal generation with different decoders.} We report FID (\(\downarrow\)). Kandinsky and Stable UnCLIP, evaluated in self-reconstruction by feeding image embeddings to the decoder (marked \(*\)), serve as upper-bound references. Our method consistently outperforms both baselines.}
\centering
\small
\setlength{\tabcolsep}{4pt}
\begin{tabular}{llllll}
\toprule
Decoder & Method~~~~~~~ & T2I $\downarrow$~~~~~~~ & A2I $\downarrow$~~~~~~~ &  $(\mathrm{T}{+}\mathrm{A}){\to}\mathrm{I}$ $\downarrow$ & Avg. $\downarrow$~~~~~~~ \\
\midrule
\multirow{4}{*}{Kandinsky}
 & Kandinsky & - & - & - & $32.99^{\ast}$ \\
 & GRAM & 62.11 & 106.97 & 92.63 & 87.23 \\
& \cellcolor{lightblue}Ours (Geodesic)  & \cellcolor{lightblue}\textbf{45.35} & \cellcolor{lightblue}\textbf{50.75} & \cellcolor{lightblue}\textbf{48.19} & \cellcolor{lightblue}\textbf{48.09} \\
 & \cellcolor{lightblue}Ours (Euclidean) & \cellcolor{lightblue}\textbf{42.72} & \cellcolor{lightblue}\textbf{50.51} & \cellcolor{lightblue}\textbf{46.56} & \cellcolor{lightblue}\textbf{46.60} \\
\midrule
\multirow{5}{*}{Stable UnCLIP}
 & Stable UnCLIP & - & - & - & $34.61^{\ast}$ \\
 & ImageBind & 50.17 & 53.59 & 46.81 & 50.19 \\
 & GRAM & 45.53 & 55.40 & 47.15 & 49.36 \\
 & \cellcolor{lightblue}Ours (Geodesic) & \cellcolor{lightblue}\textbf{39.88} & \cellcolor{lightblue}\textbf{40.16} & \cellcolor{lightblue}\textbf{40.80} & \cellcolor{lightblue}\textbf{40.23} \\
 & \cellcolor{lightblue}Ours (Euclidean) & \cellcolor{lightblue}\textbf{39.63} & \cellcolor{lightblue}\textbf{39.95} & \cellcolor{lightblue}\textbf{41.03} & \cellcolor{lightblue}\textbf{40.20} \\
\bottomrule
\end{tabular}
\label{table:generation}
\end{table*}


\begin{figure*}[t]
\vskip -1mm
  \centering  \includegraphics[width=1\textwidth]{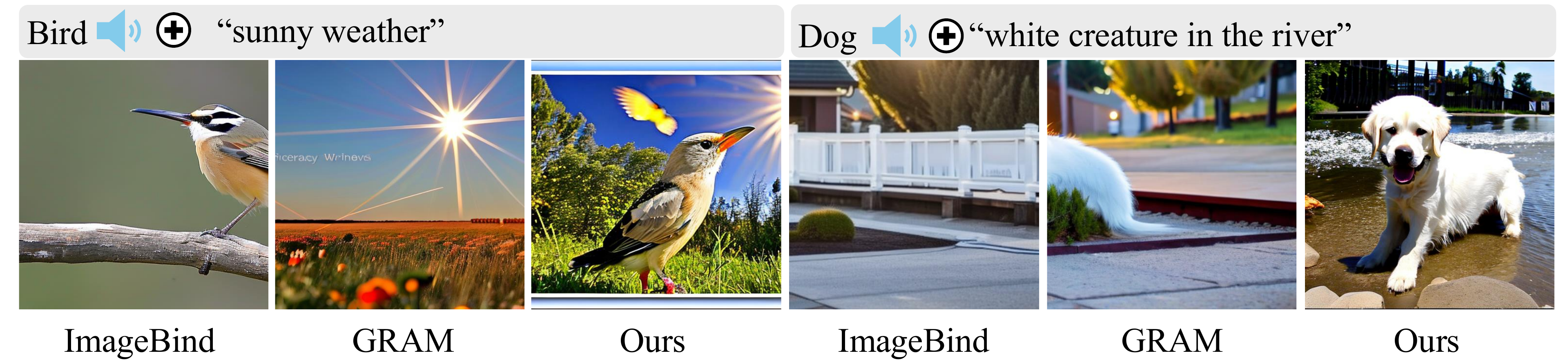} 
\caption{\textbf{Modality-interpolation generation results (T+A) $\to$ I.} When interpolating between text and audio representations, our method has a better ability to fuse the semantic information across modalities, leading to better generation. } 
\label{fig:fig-inter}
\vskip -4.5mm
\end{figure*}


To evaluate the distributional modality gap, we use a simple proxy: if multiple modalities are {well} aligned to a shared embedding distribution, embeddings from non-image modalities (e.g., audio or text) should be seamlessly decodable by an image generator trained on image embeddings. Under this view, cross-modal generation quality provides a direct indicator of cross-modal alignment. Accordingly, we adopt UnCLIP-style generators~\citep{ramesh2022hierarchical}, which use a separate image generator trained on image embeddings.

\noindent{\textbf{Dataset.}} We use the VGGSound dataset~\citep{chen2020vggsound} to evaluate the generation performance. VGGSound is an audio-visual correspondent dataset, allowing us to build a semantically aligned vision-audio-text triplet. It has around 200K video clips, annotated with 309 sound classes.
The dataset does not provide the video caption. Hence, we use the captioner provided by VAST~\citep{chen2023vast} to generate video captions. 
{We select 1024 videos for testing and utilize the remaining ones for training.}

\noindent{\textbf{Experimental setting.}} 
We map all modalities into a shared, image-anchored embedding space and evaluate two encoder-decoder configurations compatible with UnCLIP-style generators. (i) \emph{ViT-H}: CLIP ViT-H/14 image-text encoders paired with the ImageBind audio encoder, compatible with the Stable UnCLIP decoder~\citep{ramesh2022hierarchical}. (ii) \emph{ViT-bigG}: CLIP ViT-bigG-14 vision-text encoders combined with the BEATs audio encoder~\citep{chen2022beats}, compatible with the Kandinsky decoder~\citep{razzhigaev2023kandinsky}. For comparison, we re-train GRAM and use the released ImageBind weights pretrained on large-scale data. We evaluate text-to-image (T2I), audio-to-image (A2I), and modality interpolation {generation} using Fr'echet Inception Distance (FID)~\citep{heusel2017gans}.


\noindent{\textbf{T-SNE visualization.}} 
We visualize the joint embedding space using 2D t-SNE (Fig.~\ref{fig:tsne-vis}) to illustrate modality gaps under different training objectives. 
We extract text, vision, and audio embeddings from VGGSound and compute t-SNE. As shown in Fig.~\ref{fig:tsne-vis}, training with an InfoNCE-type objective {(both paired and volume InfoNCE)} yields clearly separated, modality-specific clusters (i.e., modality distribution gap), whereas our method produces substantially tighter cross-modal alignment, 
{significantly reducing the distribution gap.}

\noindent{\textbf{Results.}} We compare against GRAM and ImageBind using both \emph{Kandinsky} and \emph{Stable UnCLIP} decoders. When conditioned on image embeddings, these decoders achieve FID \(32.99\) and \(34.61\), respectively, serving as dataset-specific upper bounds (decoder self-reconstruction). UniAlign yields substantial improvements in cross-modal generation for both text-to-image (T2I) and audio-to-image (A2I) over InfoNCE-trained baselines. As shown in Table~\ref{table:generation}, UniAlign significantly outperforms GRAM with \emph{Kandinsky}, and improves over ImageBind and GRAM by about \(10\) FID with \emph{Stable UnCLIP}. These gains are consistent across architectures and decoders, indicating the robustness of our objective and suggesting a reduced modality gap via tighter distributional alignment. We also evaluate a geodesic-kernel variant, which performs on par with the Euclidean version, further supporting the generality of our principle. Additional qualitative results are provided in Appendix~\ref{app:experiments}.



\noindent\textbf{Modality interpolation.}
If {multiple} modalities are aligned to a {shared} distribution, a straightforward application is \emph{embedding interpolation}, which blends information from different modalities directly in the shared space for image synthesis (as opposed to conditioning a generator via cross-attention from a single modality).
Prior work has primarily demonstrated this for vision–language with DALL\!-\!E~2~\citep{ramesh2022hierarchical}.
Here, we go beyond vision-language interpolation for generation. We interpolate modality embeddings (e.g., \((T{+}A)/2\) and generate images with Kandinsky and Stable UnCLIP decoders.
Our method outperforms baselines both quantitatively (Table.~\ref{table:generation} with lower FID) and qualitatively (Fig.~\ref{fig:fig-inter}), indicating an improved ability to fuse complementary semantics across modalities.
We attribute these gains to the reduced cross-modal distribution gap and the resulted smoothness of the shared embedding manifold.

\section{Conclusion}
We introduced a conflict-aware principle for multimodal representation learning that decouples \emph{uniformity} from \emph{alignment}, overcoming key limitations of InfoNCE when modality number \(M\!\ge\!3\). By promoting intra-modality uniformity and anchoring positive alignment, our method directly reduces cross-modal distribution gaps. A divergence-based analysis further shows that these objectives serve as tractable estimators for minimizing a global discrepancy, providing theoretical guarantees. Empirically, the learned embeddings achieve strong performance in video retrieval and cross-modal generation with UnCLIP decoders, while t-SNE visualizations confirm improved modality integration. Overall, our approach offers a conflict-free and theoretically grounded framework for unifying discriminative and generative multimodal tasks without task-specific modules.

\noindent{\textbf{Limitation.}} Our study mainly focuses on fine-tuning rather than large-scale pretraining. Extending the proposed principle to large-scale pretraining would require substantial compute resources and is left for future work.

\section*{Impact Statement}
This paper presents work whose goal is to advance the field of machine learning. There are many potential societal consequences of our work, none of which we feel must be specifically highlighted here.

\bibliography{icml_reference}
\bibliographystyle{icml2026}

\newpage
\appendix
\onecolumn

\section{Proof of Proposition~\ref{cor:alignment-uniformity-conflict}}
\label{app:alignment-uni-conflict}
\begin{customproposition}{1}[Alignment–uniformity Conflict]
\label{cor:alignment-uniformity-conflict}
Let $\bm{\Phi}_a=\sum_{n\neq a}\bm{\Phi}_a^{(n)}$ be the total uniformity force on anchor $a$. Assume each per-modality component admits the decomposition:
\begin{equation}
\bm{\Phi}_a^{(n)} \;=\; c_n\,\hat{\mathbf{V}}_a \;+\; \bm{\varepsilon}_n.
\end{equation}

Here, $\hat{\mathbf{V}}_a \triangleq \mathbf{V}_a/\|\mathbf{V}_a\|$ is the direction of the total alignment force. The scalar $c_n \triangleq \bm{\Phi}_a^{(n)} \cdot \hat{\mathbf{V}}_a$ quantifies the magnitude of systematic conflict from modality $n$ in this direction, and satisfies $c_n \ge c_0 $ for some positive constant $c_0$. The vector $\bm{\varepsilon}_n$ is a random perturbation unique to modality $n$, assumed to be zero-mean, mutually independent, and with bounded covariance.

Under these assumptions, the overall alignment-uniformity conflict $\zeta_a$ converges to its maximum value as the number of modalities $M$ increases:
\begin{equation}
\mathbb{E}[\zeta_a]
=\mathbb{E}\!\left[\cos\!\big(\mathbf{V}_a,\bm{\Phi}_a\big)\right]
\ \rightarrow \ 1 \quad \text{as } M\to\infty.
\end{equation}
\end{customproposition}

\begin{proof}
The proof proceeds in three stages. First, we provide a formal justification for the decomposition of the per-modality uniformity force. Second, we derive a precise expression for the conflict metric $\zeta_a$ based on this decomposition. Finally, we analyze the asymptotic behavior of this expression as the number of modalities $M \to \infty$.

\paragraph{Justification of the Decomposition} The decomposition of $\bm{\Phi}_a^{(n)}$ is a formalization of the geometric principle of orthogonal projection. For any vector $\bm{\Phi}_a^{(n)}$ and a given direction defined by the unit vector $\hat{\mathbf{V}}_a$, we can uniquely decompose $\bm{\Phi}_a^{(n)}$ into a component parallel to $\hat{\mathbf{V}}_a$ and a component orthogonal to it. The component parallel to $\hat{\mathbf{V}}_a$ is its orthogonal projection, which we define as the {systematic component}:
\begin{equation}
\text{Proj}_{\hat{\mathbf{V}}_a}(\bm{\Phi}_a^{(n)}) = (\bm{\Phi}_a^{(n)} \cdot \hat{\mathbf{V}}_a) \hat{\mathbf{V}}_a.
\end{equation}

Intuitively, in each modality, non-paired but semantically similar samples
(``hard negatives'') exert a weak but systematic pull in the same direction as the
true cross-modal target; this shared component is modeled by $c_n\hat{\mathbf{V}}_a$.
Residual variation due to batch composition, data augmentations, and encoder
stochasticity is captured by zero-mean, bounded perturbations $\bm{\varepsilon}_n$
that are approximately independent across modalities. As the number of modalities
increases, the systematic components add coherently while the residuals average
out, leading to the observed accumulation of alignment–uniformity conflict.

The conflict metric $\zeta_a$ is the cosine similarity between $\mathbf{V}_a$ and the total uniformity force $\bm{\Phi}_a$:
\begin{equation}
\zeta_a = \cos(\mathbf{V}_a, \bm{\Phi}_a) = \frac{\mathbf{V}_a \cdot \bm{\Phi}_a}{\|\mathbf{V}_a\| \|\bm{\Phi}_a\|} = \frac{\hat{\mathbf{V}}_a \cdot \bm{\Phi}_a}{\|\bm{\Phi}_a\|}.
\end{equation}
Let $N = M-1$. The total uniformity force is $\bm{\Phi}_a = \sum_{n=1}^{N} \bm{\Phi}_a^{(n)} = (\sum_{n=1}^{N} c_n) \hat{\mathbf{V}}_a + \sum_{n=1}^{N} \bm{\varepsilon}_n$. Let $S_c = \sum_{n=1}^{N} c_n$ and $\mathbf{S}_\varepsilon = \sum_{n=1}^{N} \bm{\varepsilon}_n$. Due to orthogonality, the numerator of $\zeta_a$ is $\hat{\mathbf{V}}_a \cdot \bm{\Phi}_a = S_c$ and the squared norm of the denominator is $\|\bm{\Phi}_a\|^2 = S_c^2 + \|\mathbf{S}_\varepsilon\|^2$. Substituting these back, we obtain a precise expression for $\zeta_a$:
\begin{equation}
\zeta_a = \frac{S_c}{\sqrt{S_c^2 + \|\mathbf{S}_\varepsilon\|^2}} = \frac{1}{\sqrt{1 + \frac{\|\mathbf{S}_\varepsilon\|^2}{S_c^2}}}.
\end{equation}

The proof now hinges on showing that the ratio $\frac{\|\mathbf{S}_\varepsilon\|^2}{S_c^2}$ converges to zero as $N \to \infty$. The denominator $S_c^2 = (\sum c_n)^2 \ge (N c_0)^2$ grows at least quadratically. For the numerator, due to the independence and zero-mean properties of $\{\bm{\varepsilon}_n\}$, its expected value grows at most linearly:
\begin{equation}
\mathbb{E}[\|\mathbf{S}_\varepsilon\|^2] = \sum_{n=1}^N \mathbb{E}[\|\bm{\varepsilon}_n\|^2] \le N C_\varepsilon,
\end{equation}
for some constant $C_\varepsilon < \infty$ implied by the bounded covariance. The ratio of the expected numerator to the lower-bounded denominator is of the order $O(N)/O(N^2) = O(1/N)$, which converges to 0. This implies that the random variable $\frac{\|\mathbf{S}_\varepsilon\|^2}{S_c^2}$ converges to 0 in probability.

By the Continuous Mapping Theorem, $\zeta_a$ converges in probability to $1$. As $\zeta_a$ is bounded in $[-1, 1]$, the Bounded Convergence Theorem ensures that convergence in probability to a constant implies convergence in expectation. Therefore:
\begin{equation}
\lim_{M\to\infty} \mathbb{E}[\zeta_a] = 1.
\end{equation}
\end{proof}


\section{Proof of Proposition~\ref{cor:intra-alignmen-conflict}}
\label{app:proof-intra-alignment-coro}

\begin{customproposition}{2}[Intra-alignment Conflict]
\label{cor:intra-alignmen-conflict}

The expected intra-alignment conflict, $\mathbb{E}[\chi_a]$, is governed by $M$ and the average pairwise alignment \(\bar{\mu} = \mathbb{E}[\mathbf{z}_i^{(m)\top} \mathbf{z}_i^{(n)}] \in [0, 1] \) for \(m \neq n\) between modalities:
\begin{equation}
\label{eq:conflict_scaling_law}
\mathbb{E}[\chi_a] \ \ge\ 1 - \sqrt{\frac{1+(M-2)\bar{\mu}}{M-1}}.
\end{equation}
For imperfect alignment (\(\bar{\mu} < 1\)), the conflict increases with the number of modalities \(M\) and admits a non-zero asymptotic lower bound:
\begin{equation}
\liminf_{M \to \infty} \mathbb{E}[\chi_a] \ \ge\ 1 - \sqrt{\bar{\mu}}.    
\end{equation}

\end{customproposition}

\begin{proof}
The intra-alignment conflict for an anchor modality $a$ is defined as:
\[\chi_a \triangleq 1-\frac{\|\mathbf{V}_a\|_2}{\sum_{n\neq a} w_{an}/\tau}\]
where the alignment force is $\mathbf{V}_a = \sum_{n\neq a} \frac{w_{an}}{\tau} \mathbf{z}_i^{(n)}$. To derive the fundamental scaling relationship with the number of modalities $M$, we make a simplifying assumption of uniform weighting, i.e., $w_{an}/\tau = 1$ for all $n \neq a$. Under this assumption,
\begin{equation}
\chi_a = 1 - \frac{\|\mathbf{V}_a\|_2}{M-1}, 
\qquad 
\mathbf{V}_a = \sum_{n\neq a} \mathbf{z}_i^{(n)}.
\end{equation}
Let $N=M-1$. Then
\begin{equation}
\|\mathbf{V}_a\|_2^2 
= \sum_{n=1}^N \sum_{m=1}^N \mathbf{z}_i^{(n)\top}\mathbf{z}_i^{(m)}
= N + \sum_{n\neq m}\mathbf{z}_i^{(n)\top}\mathbf{z}_i^{(m)}.
\end{equation}
Taking expectations and using $\bar{\mu}=\mathbb{E}[\mathbf{z}_i^{(n)\top}\mathbf{z}_i^{(m)}]$ for $n\neq m$,
\begin{equation}
\mathbb{E}\big[\|\mathbf{V}_a\|_2^2\big]
= N + N(N-1)\bar{\mu}
= (M-1)\big(1+(M-2)\bar{\mu}\big).
\end{equation}
By Jensen's inequality, 
\begin{equation}
\mathbb{E}\big[\|\mathbf{V}_a\|_2\big]
\ \le\
\sqrt{\mathbb{E}\big[\|\mathbf{V}_a\|_2^2\big]}
=
\sqrt{(M-1)\big(1+(M-2)\bar{\mu}\big)}.
\end{equation}
Hence
\begin{equation}
\mathbb{E}[\chi_a]
= 1 - \frac{\mathbb{E}\|\mathbf{V}_a\|_2}{M-1}
\ \ge\
1 - \sqrt{\frac{1+(M-2)\bar{\mu}}{M-1}},
\end{equation}
which proves \eqref{eq:conflict_scaling_law}. Finally,
\begin{equation}
\lim_{M\to\infty} \sqrt{\frac{1+(M-2)\bar{\mu}}{M-1}}
= \sqrt{\bar{\mu}}
\quad\Rightarrow\quad
\liminf_{M\to\infty}\mathbb{E}[\chi_a]\ \ge\ 1-\sqrt{\bar{\mu}}.
\end{equation}
\end{proof}

\section{Generalized Hölder Divergence}
\label{app:divergence}

\paragraph{KDE estimation of the global Hölder divergence.}
Let $\{p_m(\mathbf{z})\}_{m=1}^M$ be the (unknown) continuous densities of $M$ modalities and
\begin{equation}
\begin{aligned}
D_{\mathrm{H\ddot{o}lder}}
\;&=\;
-\log\frac{\int \prod_{m=1}^M p_m(\mathbf{z})\,d\mathbf{z}}
{\left(\prod_{m=1}^M \int p_m(\mathbf{z})^M\,d\mathbf{z}\right)^{1/M}}
\;\\
&=\;
\frac{1}{M}\sum_{m=1}^M \log\!\left(\int p_m(\mathbf{z})^M\,d\mathbf{z}\right)
\;-\;\log\!\left(\int \prod_{m=1}^M p_m(\mathbf{z})\,d\mathbf{z}\right),
\label{eq:holder-div-app}
\end{aligned}
\end{equation}
which is nonnegative by Hölder’s inequality and equals $0$ iff the Hölder inequality is tight.

\medskip

For modality $m$, let $\{\mathbf{z}_k^{(m)}\}_{k=1}^B$ be a batch of embeddings on $\mathbb{R}^d$ and define the kernel density estimator (KDE)
\begin{equation}
\widehat{p}_m(\mathbf{z})
\;=\;
\frac{1}{B}\sum_{k=1}^B K_\tau\!\big(\mathbf{z},\mathbf{z}_k^{(m)}\big),
\qquad
K_\tau(\mathbf{z},\mathbf{z}')
\;=\;
\frac{1}{(2\pi\tau^2)^{d/2}}\,
\exp\!\Big(-\tfrac{\|\mathbf{z}-\mathbf{z}'\|_2^2}{2\tau^2}\Big),
\label{eq:kde-def}
\end{equation}
with bandwidth $\tau>0$.\footnote{If one uses the \emph{unnormalized} kernel 
$\kappa(\mathbf{z},\mathbf{z}')=\exp(-\|\mathbf{z}-\mathbf{z}'\|_2^2/(2\tau^2))$,
then $\widehat{p}_m$ is scaled by a constant depending on $(d,\tau)$.
This yields an \emph{additive constant} in $D_{\mathrm{H\ddot{o}lder}}$ that does not affect optimization; we drop such constants in practice.}

\medskip

Using $\int p_m^M = \mathbb{E}_{Z\sim p_m}\!\big[p_m(Z)^{M-1}\big]$ and
$\int\prod_{m=1}^M p_m = \mathbb{E}_{Z\sim p_1}\!\big[\prod_{m=2}^M p_m(Z)\big]$,
we obtain Monte-Carlo plug-in estimators by sampling from the empirical $p_m$ via $\{\mathbf{z}_j^{(m)}\}$ and evaluating the KDEs:
\begin{align}
\int \widehat{p}_m(\mathbf{z})^M\,d\mathbf{z}
&=
\mathbb{E}_{Z\sim \widehat{p}_m}\!\big[\widehat{p}_m(Z)^{M-1}\big]
\;\approx\;
\frac{1}{B}\sum_{j=1}^B
\Biggl(\frac{1}{B}\sum_{k=1}^B K_\tau\!\big(\mathbf{z}_j^{(m)},\mathbf{z}_k^{(m)}\big)\Biggr)^{M-1},
\label{eq:kde-marg-final}
\\
\int \prod_{m=1}^M \widehat{p}_m(\mathbf{z})\,d\mathbf{z}
&=
\mathbb{E}_{Z\sim \widehat{p}_1}\!\Big[\prod_{m=2}^M \widehat{p}_m(Z)\Big]
\;\approx\;
\frac{1}{B}\sum_{j=1}^B \prod_{m=2}^M
\Biggl(\frac{1}{B}\sum_{k=1}^B K_\tau\!\big(\mathbf{z}_j^{(1)},\mathbf{z}_k^{(m)}\big)\Biggr).
\label{eq:kde-joint-final}
\end{align}
Equivalently, with the unnormalized Gaussian kernel $\kappa$ (dropping constants), the formulas above match
\begin{equation}
\int p_m^M \approx \frac{1}{B}\sum_{j=1}^B \Bigl(\frac{1}{B}\sum_{k=1}^B \kappa(\mathbf{z}_j^{(m)},\mathbf{z}_k^{(m)})\Bigr)^{M-1},
\qquad
\int \prod_{m=1}^M p_m \approx \frac{1}{B}\sum_{j=1}^B \prod_{m=2}^M \Bigl(\frac{1}{B}\sum_{k=1}^B \kappa(\mathbf{z}_j^{(1)},\mathbf{z}_k^{(m)})\Bigr).
\end{equation}




Define the (within-modality) kernel means
$s_i^{(m)} \triangleq \frac{1}{B}\sum_{k=1}^B K_\tau(\mathbf{z}_i^{(m)},\mathbf{z}_k^{(m)})$
and the (cross-modality-to-anchor) kernel means
$c_i \triangleq \prod_{m=2}^M \frac{1}{B}\sum_{k=1}^B K_\tau(\mathbf{z}_i^{(1)},\mathbf{z}_k^{(m)})$.
Then the Hölder divergence estimator (up to an additive constant if using $\kappa$) is
\begin{equation}
\widehat{D}_{\mathrm{H\ddot{o}lder}}
\;=\;
\frac{1}{M}\sum_{m=1}^M
\log\!\Bigg(\frac{1}{B}\sum_{i=1}^B \big(s_i^{(m)}\big)^{\,M-1}\Bigg)
\;-\;
\log\!\Bigg(\frac{1}{B}\sum_{i=1}^B c_i\Bigg).
\label{eq:holder-kde-estimator}
\end{equation}
All terms are differentiable; the computation costs $O(MB^2)$ and can be vectorized via kernel matrices
$K^{(m)}_{ij}=K_\tau(\mathbf{z}_i^{(m)},\mathbf{z}_j^{(m)})$ and
$K^{(m\to 1)}_{ij}=K_\tau(\mathbf{z}_i^{(1)},\mathbf{z}_j^{(m)})$.

\section{Design Space for Uniformity and Alignment}
\label{app:design-space}

\begin{table}[t]
\centering
\caption{\textbf{Design space for uniformity and alignment}. Uniformity can be instantiated in Euclidean or manifold geometries; alignment can incorporate geometric constraints beyond pairwise distance.}
\label{tab:uniformity-alignment-design}
\resizebox{\linewidth}{!}{%
\begin{tabular}{@{}l l m{0.35\linewidth} m{0.35\linewidth}@{}}
\toprule
\textbf{Principle} & \textbf{Space} & \textbf{Kernel/Metric} & \textbf{Notes} \\
\midrule
\multirow{2}{*}{\begin{tabular}[c]{@{}l@{}}Uniformity\\(repulsion)\end{tabular}}
& Euclidean $(\mathbb{R}^d)$
& $
\exp\!\Bigl(-{\|\mathbf{z}_i^{(m)}-\mathbf{z}_j^{(m)}\|_2^2}/{2\tau^2}\Bigr)
$
& Gaussian kernel in $\mathbb{R}^d$; encourages spread. \\[0.8em]

& Unit Hypersphere $(\mathbb{S}^{d-1})$
& $\displaystyle
\exp\!\Bigl(-{d_{\mathbb{S}}(\mathbf{z}_i^{(m)},\mathbf{z}_j^{(m)})^2}/{2\tau^2}\Bigr)$
& Geodesic (Riemannian) Gaussian on $\mathbb{S}^{d-1}$. \\

\midrule
\multirow{2}{*}{\begin{tabular}[c]{@{}l@{}}Alignment\\(attraction)\end{tabular}}
& Euclidean $(\mathbb{R}^d)$
& $\displaystyle
\|\mathbf{z}_i^{(m)}-\mathbf{z}_i^{(n)}\|_2^2
$
& Pairwise matching per sample. \\[0.8em]

& Unit Hypersphere $(\mathbb{S}^{d-1})$
& $\displaystyle
\bigl[d_{\mathbb{S}}(\mathbf{z}_i^{(m)},\mathbf{z}_i^{(n)})\bigr]^2
$
& Geodesic pairwise alignment. \\[0.8em]
& Area/volume preservation
& $\displaystyle
\sqrt{\det \mathbf{G}(\mathbf{z}^{(1)},\ldots,\mathbf{z}^{(M)}}
$
& Penalizes global volume; $
\mathbf{G}$ is the Gram Matrix. \\
\bottomrule
\end{tabular}%
}
\label{table:principle}
\end{table}

We illustrate some possible designs following our principle in Table~\ref{table:principle}, showing the generality and flexibility of our framework.

\section{Computational Complexity}

Our framework is formulated upon the core principles of feature alignment and uniformity, which are fundamental to the contrastive learning objective. The implementation of our method operates within the standard computational pipeline of modern contrastive learning. For a given batch of size \(B\) with \(d\)-dimensional representations, the dominant computational cost remains the construction of the \(B \times B\) pairwise similarity matrix, an operation with \(O(B^2 d)\) time complexity. The objectives of alignment and uniformity are then computed based on this matrix, inheriting the same computational profile as the loss calculation and gradient backpropagation stages of standard contrastive methods.

Crucially, our formulation does not require any operations beyond those already present in the baseline. As such, our method introduces no additional computational overhead and shares an identical time and memory complexity profile with widely-used InfoNCE-based frameworks. This efficiency ensures our approach is scalable and readily applicable to large-scale training regimes.

\section{Experimental Details and Additional Experiments}
\label{app:experiments}

\subsection{Implementation Details}
\label{sec:impl}

All experiments use {4$\times$ NVIDIA A6000} GPUs. We train with {AdamW} using a learning rate of \(\mathbf{2\times10^{-5}}\) and a batch size of {128 per GPU} (global batch \(=512\)). For zero-shot video retrieval, we use the GRAM~\citep{cicchetti2024gramian} codebase and only replace GRAM's volume-based InfoNCE with our uniformity and alignment losses. We follow their settings: each video clip is sampled with {8 frames} during training, and the model is trained for {5 epochs}. For cross-modal generation, we train on {VGGSound} for {50 epochs}. For hyperparameters, we use \(\lambda_{\mathrm{align}} = 1\) and set the temperature of uniformity loss as \(\tau_{\mathrm{ctr}} = 0.07\), the same as in standard CLIP. We use a separate temperature $\tau_{\mathrm{ctr}}=0.07$ for tuple-level uniformity $U(\mathbf{C})$. 


\subsection{Empirical evidence of InfoNCE conflict}
\label{app:evidence-conflict}

\begin{figure*}[htbp]
  \centering  
  \includegraphics[width=0.85\textwidth]{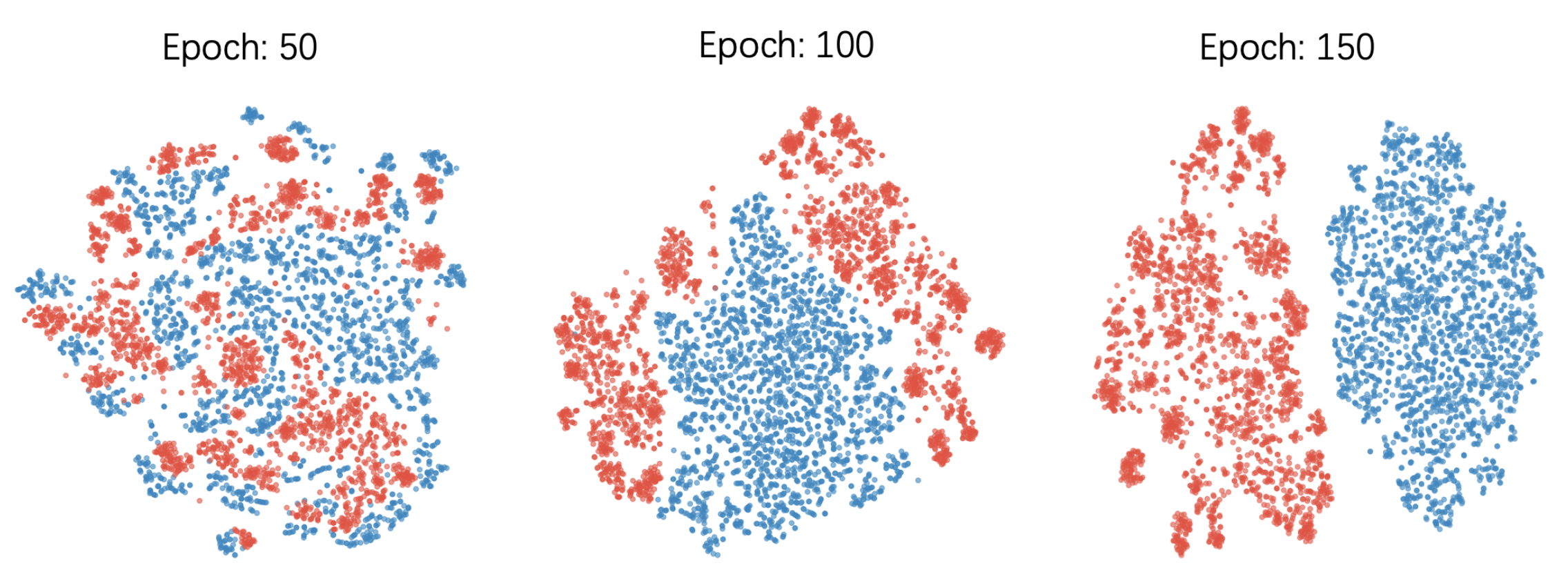} 
\caption{\textbf{t-SNE visualization of image and text embeddings across training epochs when optimizing an InfoNCE-based text-to-image binding.} The two modalities move closer at early stages but later separate again, suggesting inconsistent distributional alignment.} 
\label{fig:evidence-conflict}
\end{figure*}

To further support the theoretical analysis of the InfoNCE conflict and to motivate our method, we provide an empirical observation from training dynamics.
Specifically, we train a lightweight 5-layer Transformer adapter on top of the CLIP ViT-L, optimized with the standard InfoNCE objective.
Following the ImageBind-style setup, we treat the CLIP image embeddings as fixed targets and learn to bind text embeddings to the image embedding space.
We train on the MSCOCO training set and visualize the learned representations using t-SNE on 5K image-text pairs from the validation split, saved every 50 epochs.
As shown in Fig.~\ref{fig:evidence-conflict}, the two modalities become closer at early stages of training, but the clusters later separate again.
While t-SNE is qualitative, this phenomenon suggests that InfoNCE may not consistently reduce the distributional mismatch between modalities and can even widen the modality gap at later stages.
This observation is consistent with the analysis in \citet{yin2025distributional}, which links weakened alignment to the evolution of the (learnable) temperature during training.

\subsection{More Ablation Studies}
\label{app:ablation}

\noindent{\textbf{Sensitivity of temperature in $U(\mathbf{C})$}.} 
As we mentioned above, we use the temperature $\tau=0.07$ for the uniformity loss. However, a separate temperature $\tau_{\mathrm{ctr}}$ for the tuple-level uniformity loss $U(\mathbf{C})$ could control the global separability. Hence, we perform an ablation study on this parameter. Table~\ref{label:abl-tau} shows that the average performance is robust to the temperature $\tau$, while controlling centroid $\tau$ may affect the subtask performance (T2V and V2T).  

\begin{table}[t]
\centering
\setlength{\tabcolsep}{10pt}
\caption{Ablation on centroid uniformity temperature $\tau_{\mathrm{ctr}}$ on MSR-VTT retrieval (Recall@1, \%).}
\label{tab:msrvtt_ctr_tau_r1}
\begin{tabular}{lccc}
\toprule
$\tau_{\mathrm{ctr}}$ & T2V R@1 & V2T R@1 & Avg R@1 \\
\midrule
0.01 & 56.0 & 54.4 & 55.2 \\
0.03 & 56.8 & 53.2  & 55.0  \\
0.07 & 57.4 & 53.2 & 55.3  \\
\bottomrule
\end{tabular}
\vspace{-2mm}
\label{label:abl-tau}
\end{table}

\noindent{\textbf{Sensitivity to $\lambda_{\mathrm{align}}$ and from-scratch training.}}
We study the sensitivity of the alignment weight $\lambda_{\mathrm{align}}$ on MSR-VTT using our basic instantiation, which combines modality-wise uniformity with an explicit alignment regularizer:
\begin{equation}
\label{eq:lambda_align_ablation}
\mathcal{L}
=  \sum\nolimits_{m=1}^{M} U\!\big(Z^{(m)}\big)
+ \lambda_{\mathrm{align}}\, L_{\mathrm{align}} .
\end{equation}
In this ablation, we keep all other training settings fixed and vary only $\lambda_{\mathrm{align}}$ to isolate its effect.
In addition, we conduct this experiment with a small-scale training-from-scratch setup, demonstrating that our objective remains effective without large-scale pretraining.
For evaluation, we perform cross-modal retrieval using cosine similarity between normalized embeddings and report standard MSR-VTT retrieval metrics.
Overall, the results show that the proposed objective is stable across a broad range of $\lambda_{\mathrm{align}}$ and exhibits the expected trade-off: increasing $\lambda_{\mathrm{align}}$ strengthens cross-modal alignment, while overly large values can reduce within-modality uniformity and degrade retrieval performance.

\begin{table}[htbp]
\centering
\caption{\textbf{Sensitivity of $\lambda_{\mathrm{align}}$}.}
\label{tab:lambda_align}
\begin{tabular}{lccc}
\toprule
$\lambda_{\mathrm{align}}$ & T2V & V2T & Avg \\
\midrule
0.2 &	23.8	&26.8	& 25.3 \\
0.5 & 26.2 & 29.4 & 27.8 \\
1.0 & 23.1 & 28.6 & 25.85 \\
1.5 & 19.6 & 19.3 & 19.45 \\
2.0 & 13.1 & 13.6 & 13.35 \\
\bottomrule
\end{tabular}
\end{table}

\noindent{\textbf{$D_{\mathrm{H\ddot{o}lder}}$ curve during training.}}
To demonstrate that our principle serves as an efficient proxy for global distribution estimation, we additionally track the evolution of the cross-modal divergence $D_{\mathrm{H\ddot{o}lder}}$ on MSR-VTT throughout training.
Concretely, we train the model under the same setting as in the cross-modal generation experiments, and every 5 epochs we compute $D_{\mathrm{H\ddot{o}lder}}$ between the image, text, and audio embedding distributions using the KDE estimator in Eq.~\eqref{eq:holder-kde-estimator}.
Fig.~\ref{fig:divergence_curve} shows that the global Hölder divergence decreases rapidly during training as our method effectively closes the distribution gaps among the three modalities.
Notably, $D_{\mathrm{H\ddot{o}lder}}$ quickly approaches zero, suggesting that directly optimizing this divergence can be problematic in practice (e.g., due to early saturation and weak gradients), motivating our proxy objective instead. 

\begin{figure*}[htbp]
  \centering  \includegraphics[width=0.5\textwidth]{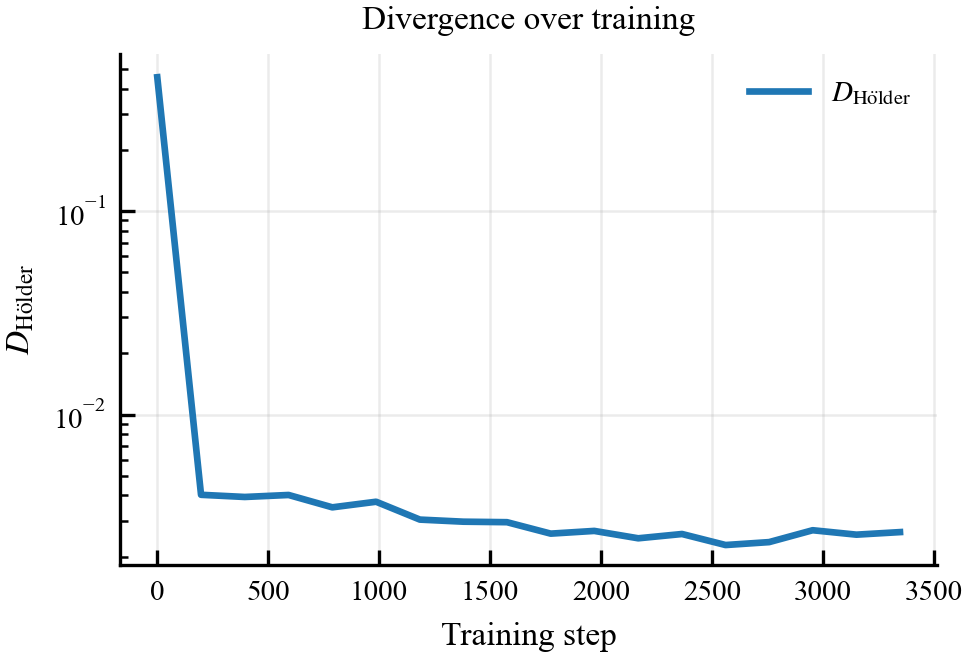} 
\caption{\textbf{$D_{\mathrm{H\ddot{o}lder}}$ curve during training.}} 
\label{fig:divergence_curve}
\end{figure*}

\subsection{Experimental Results on Modality Interpolation}

\begin{figure*}[htbp]
  \centering  \includegraphics[width=1\textwidth]{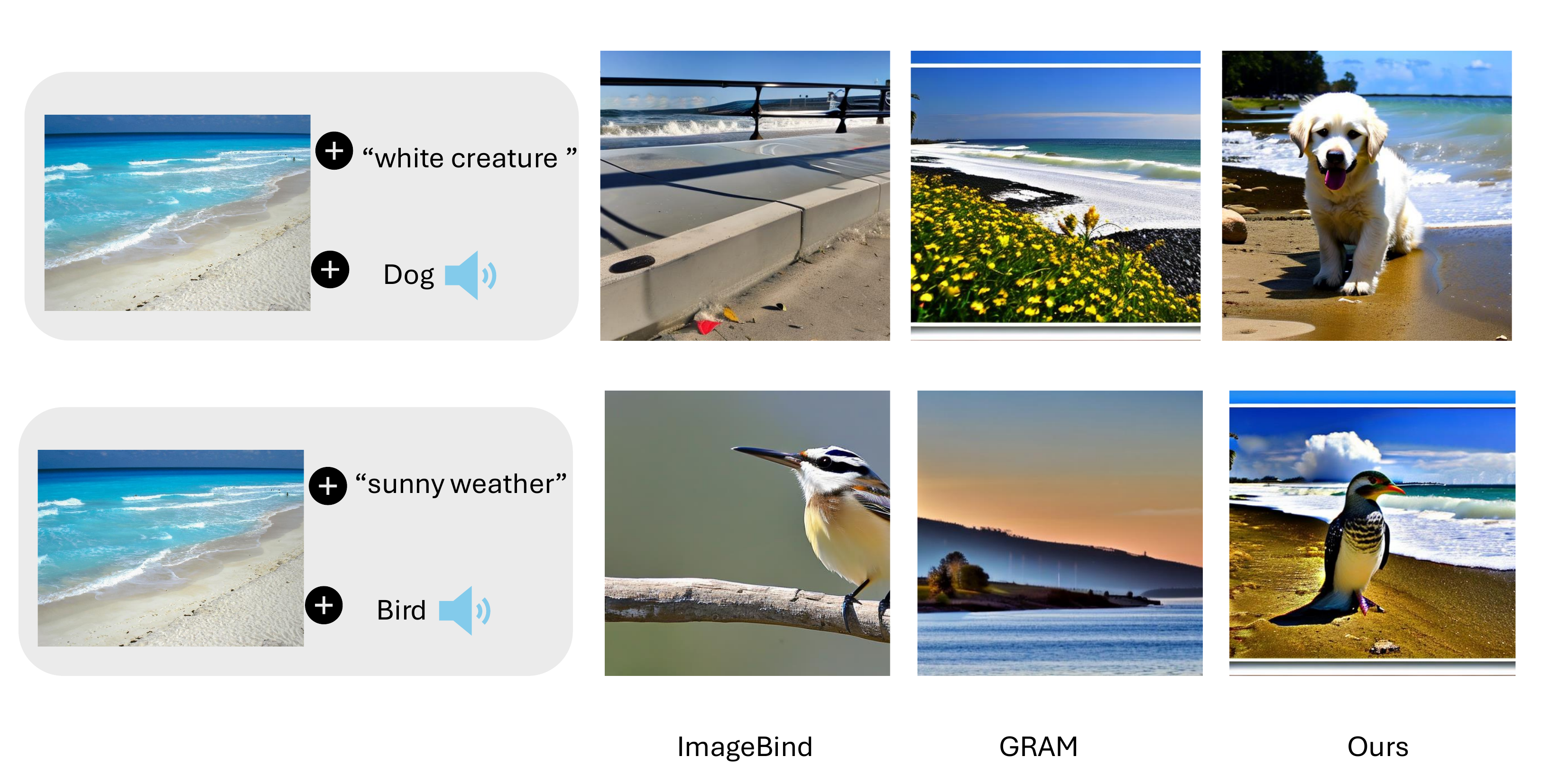} 
\caption{\textbf{Modality-interpolation generation results (V+T+A) $\to$ I.} When interpolating, vision, text, and audio representations, our method has a better ability to fuse the semantic information across modalities, leading to better generation. } 
\label{fig:fig-inter-tri}
\end{figure*}

Beyond the bimodal cases in Fig.~\ref{fig:fig-inter}, we present tri-modal interpolation results. Conditioning jointly on an image embedding, a text prompt, and an audio embedding, our model synthesizes images that integrate complementary semantics from all three modalities, demonstrating effective cross-modal fusion.

\subsection{Generation Results of VGGSound}

We present more generated samples from VGGSound in Fig.~\ref{fig:more-vggsound}. Note that the image quality of VGGSound's videos is quite noisy, making the generation results similar. Also, we adopt a raw generation process for this demonstration, where the embeddings are directly passed to the decoder without additional conditions (e.g., negative prompts or quality-enhancing constraints). This can directly reflect the goodness of the multimodal alignment without external factors. 

\begin{figure*}[htbp]
  \centering  \includegraphics[width=0.8\textwidth]{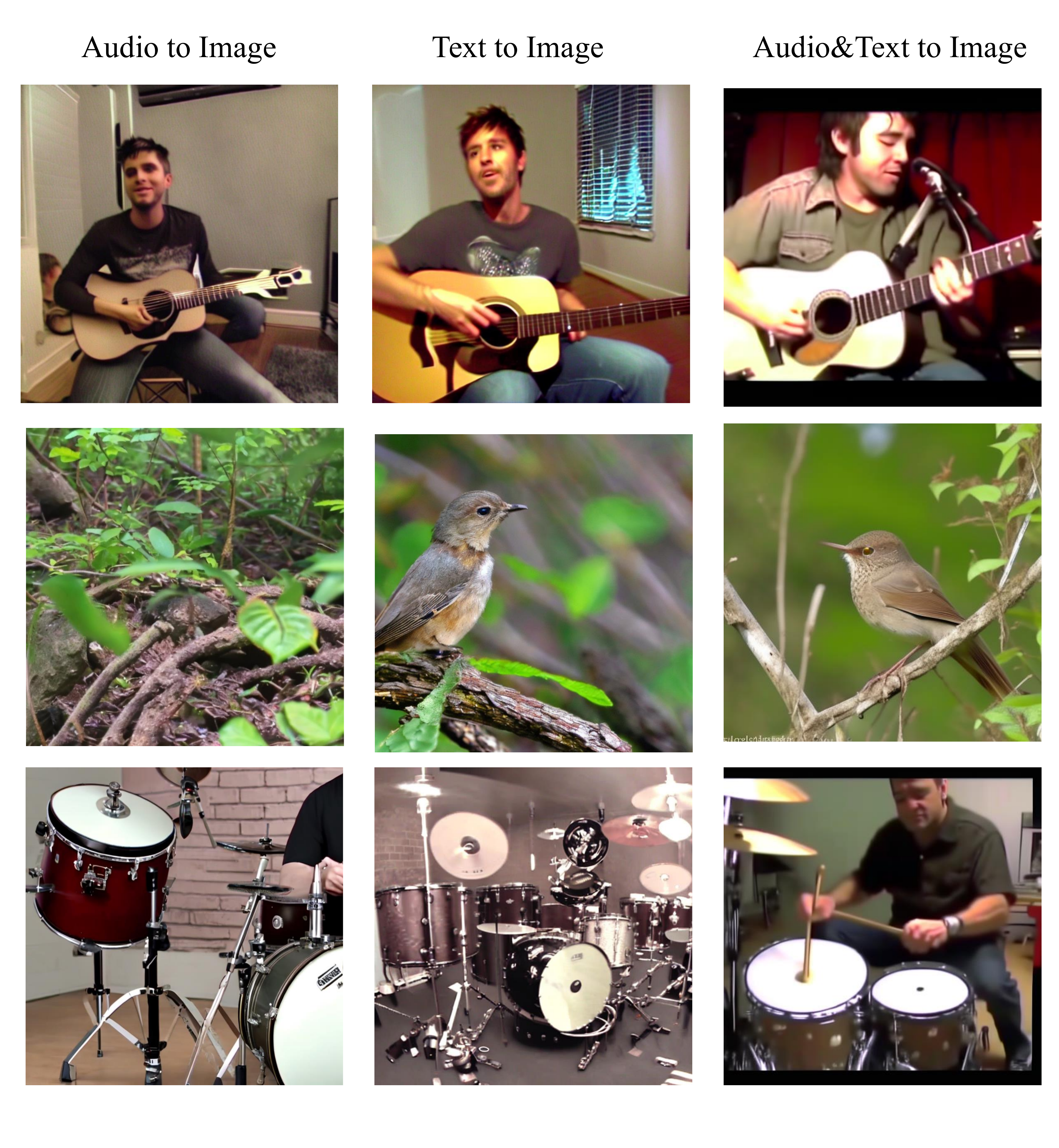} 
\caption{\textbf{More generated results from VGGsound.} We adopt a raw generation process for demonstrating the multimodal alignment ability, where the embeddings are directly passed to the decoder without additional conditions (e.g., negative prompts or quality-enhancing constraints)} 
\label{fig:more-vggsound}
\end{figure*}

\end{document}